\documentclass{article}

\usepackage{microtype}
\usepackage{graphicx}
\usepackage{subfigure}
\usepackage{booktabs} 

\usepackage{hyperref}

\usepackage[accepted]{icml2023}

\usepackage{amsmath}
\usepackage{amssymb}
\usepackage{mathtools}
\usepackage{amsthm}

\usepackage[capitalize,noabbrev]{cleveref}

\theoremstyle{plain}
\newtheorem{theorem}{Theorem}[section]
\newtheorem{proposition}[theorem]{Proposition}
\newtheorem{lemma}[theorem]{Lemma}
\newtheorem{corollary}[theorem]{Corollary}

\theoremstyle{definition}
\newtheorem{definition}[theorem]{Definition}

\theoremstyle{remark}

\usepackage[textsize=tiny]{todonotes}

\def\paperdraftmode{1} 
\newcommand{\draftonly}[1]{\ifx\paperdraftmode\undefined\else{#1}\fi}

\renewcommand{\paragraph}[1]{\textbf{#1}\,\,}

\newcommand{\appendixref}[1]{\cref{#1}}
\newcommand{\Appendixref}[1]{\Cref{#1}}

\newcommand{\squeeze}{\looseness=-1}
\usepackage{xurl}  
\renewcommand{\algorithmicrequire}{\textbf{Input:}}
\renewcommand{\algorithmicensure}{\textbf{Output:}}

\newcommand{\figvspace}{\vspace*{-6mm}}

\usepackage{upgreek}

\usepackage{amsmath,amsfonts,bm}

\def\eqref#1{Eq.~(\ref{#1})}
\def\Eqref#1{Eq.~(\ref{#1})}

\def\1{\bm{1}}

\DeclareMathAlphabet{\mathsfit}{\encodingdefault}{\sfdefault}{m}{sl}
\SetMathAlphabet{\mathsfit}{bold}{\encodingdefault}{\sfdefault}{bx}{n}

\newcommand{\R}{\mathbb{R}}

\newcommand{\softmax}{\mathrm{softmax}}

\DeclareMathOperator*{\argmax}{argmax}
\DeclareMathOperator*{\argmin}{argmin}

\newcommand{\FullDerivative}[2]{\frac{\mathrm{d}{#1}}{\mathrm{d}{#2}}}
\newcommand{\tFullDerivative}[2]{\tfrac{\mathrm{d}{#1}}{\mathrm{d}{#2}}}

\newcommand{\Size}[1]{{\left|{#1}\right|}}
\newcommand{\Abs}[1]{{\left|{#1}\right|}}
\newcommand{\Set}[1]{{\left\{{#1}\right\}}}

\newcommand\expect[2]{\mathbb{E}_{#1}{\left[ {#2} \right]}}
\newcommand\dblexpect[3]{\mathbb{E}_{#1}\,\mathbb{E}_{#2}{\left[ {#3} \right]}}
\newcommand\prob[2]{\mathbb{P}_{#1}{\left[ {#2} \right]}}

\newcommand{\f}{{\bm{f}}}
\renewcommand{\S}{{\cal{S}}}

\newcommand{\Reals}{\mathbb{R}}

\newcommand{\opt}{\mathrm{opt}}

\newcommand{\Users}{{\mathcal{U}}}
\newcommand{\Items}{{\mathcal{X}}}

\newcommand{\TPP}{\mathcal{S}}
\newcommand{\LVTPP}{\TPP^\mathrm{LV}}
\newcommand{\StatelessTPP}{\TPP^\mathrm{CR}}

\newcommand{\EmpiricalRateP}[1]{{\tfrac{1}{T}\Size{S_u\left({#1};T\right)}}}
\newcommand{\method}[1]{{\fontfamily{lmtt}\selectfont{{#1}}}} 
\newcommand{\LongTermEngagement}{\mathrm{LTER}}
\newcommand{\LvExperimentAdaptiveImprovementOverMyopic}{6.35}

\newcommand{\LvExperimentAdaptiveTZeroMedium}{5}

\newcommand{\LvExperimentBarplotSelectedKappa}{0.5}
\newcommand{\LvExperimentCfNFactors}{8}
\newcommand{\LvExperimentCfNItems}{3,706}
\newcommand{\LvExperimentCfNItemsGoodreads}{87,565}
\newcommand{\LvExperimentCfNRatings}{1,000,209}
\newcommand{\LvExperimentCfNRatingsGoodreads}{3,679,076}
\newcommand{\LvExperimentCfNUsers}{6,040}
\newcommand{\LvExperimentCfNUsersGoodreads}{41,932}
\newcommand{\LvExperimentCfRmse}{0.917}

\newcommand{\LvExperimentCfTestSetPct}{70}
\newcommand{\LvExperimentCfTrainingSetPct}{30}

\newcommand{\LvExperimentConfidenceLevelPct}{95}
\newcommand{\LvExperimentControlGroupSize}{3,528}
\newcommand{\LvExperimentControlGroupSizeGoodreads}{28,652}

\newcommand{\LvExperimentControlGroupWeightPct}{70}

\newcommand{\LvExperimentEvaluationSetUsers}{1,000}

\newcommand{\LvExperimentLvAlpha}{0.065}
\newcommand{\LvExperimentLvBeta}{0.01, 0.04, 0.09, 0.16, 0.25}
\newcommand{\LvExperimentLvDelta}{0.001}
\newcommand{\LvExperimentLvGamma}{0.02}
\newcommand{\LvExperimentMainExperimentNonzeroTreatments}{0.05, 0.1, 0.15}
\newcommand{\LvExperimentMainExperimentNonzeroTreatmentsCount}{3}

\newcommand{\LvExperimentNumRepetitions}{10}

\newcommand{\LvExperimentOverallImprovementOverAdaptiveNeg}{0.377}
\newcommand{\LvExperimentOverallImprovementOverAdaptiveNegGoodreads}{0.065}
\newcommand{\LvExperimentOverallImprovementOverArgmax}{2.05}

\newcommand{\LvExperimentOverallImprovementOverArgmaxGoodreads}{5.14}
\newcommand{\LvExperimentOverallImprovementOverMyopic}{5.98}

\newcommand{\LvExperimentOverallImprovementOverOracle}{-0.791}

\newcommand{\LvExperimentOverallImprovementOverOracleGoodreads}{-0.546}
\newcommand{\LvExperimentOverallImprovementOverSafety}{5.74}

\newcommand{\LvExperimentSafetyPolicyThreshold}{16}
\newcommand{\LvExperimentSafetyPolicyThresholdLow}{14}

\newcommand{\LvExperimentSimulationLength}{100}
\newcommand{\LvExperimentSoftmaxT}{0.5}

\newcommand{\LvExperimentTotalSimulationTimeMinutes}{28}

\newcommand{\LvExperimentTreatmentGroupSize}{504}
\newcommand{\LvExperimentTreatmentGroupSizeGoodreads}{4,093}
\newcommand{\LvExperimentTreatmentGroupWeightPct}{10}

\newcommand{\LvCodeURL}{\url{https://github.com/edensaig/suggest-breaks}}
\usepackage{enumitem}
\setlist[itemize]{leftmargin=*, topsep=0pt, itemsep=1pt}

\icmltitlerunning{Learning to Suggest Breaks: Sustainable Optimization of Long-Term User Engagement}

\begin{document}

\twocolumn[
\icmltitle{
Learning to Suggest Breaks: \\
Sustainable Optimization of Long-Term User Engagement
}

\icmlsetsymbol{equal}{*}

\begin{icmlauthorlist}
\icmlauthor{Eden Saig}{technioncs}
\icmlauthor{Nir Rosenfeld}{technioncs}
\end{icmlauthorlist}

\icmlaffiliation{technioncs}{Department of Computer Science, Technion -- Israel Institute of Technology, Haifa, Israel}

\icmlcorrespondingauthor{Eden Saig}{edens@cs.technion.ac.il}
\icmlcorrespondingauthor{Nir Rosenfeld}{nirr@cs.technion.ac.il}

\icmlkeywords{Machine Learning, ICML}

\vskip 0.3in
]

\printAffiliationsAndNotice{}  

\begin{abstract}

Optimizing user engagement is a key goal for modern recommendation systems, but blindly pushing users towards increased consumption risks burn-out, churn, or even addictive habits. To promote digital well-being, most platforms now offer a service that periodically prompts users to take breaks. These, however, must be set up manually, and so may be suboptimal for both users and the system.
In this paper, we study the role of breaks in recommendation,
and propose a framework for learning optimal breaking policies that promote and sustain long-term engagement.
Based on the notion that recommendation dynamics are susceptible to both positive \emph{and} negative feedback,
we cast recommendation as a
Lotka-Volterra dynamical system,
where breaking reduces to a problem of optimal control.
We then give an efficient learning algorithm, provide theoretical guarantees, and empirically demonstrate the utility of our approach on semi-synthetic data.

\end{abstract}

\section{Introduction}
\label{sec:intro}

As consumers of content, we have come to rely extensively on algorithmic recommendations.
This has made the task of recommending relevant content
key to the success of modern media platforms. 
Recommendation systems are built with the primary goal of maximizing user engagement;
this is typically achieved by recommending on the basis of
learned predictive models,
trained to predict for each user the potential relevance of new items.
Ideally, improved predictive models should lead to increased engagement due to better and more useful recommendations.
However, with media platforms becoming more engaging, there is a growing concern about their tendency to drive users towards excessive consumption \citep{elhai2017problematic, lee2014dark}.

To preserve user well-being,
most major platforms now provide a service that periodically suggest taking ``breaks''
\citep{constine2018instagram, perez2018apple},
with the idea that occasional intentional disruptions 
can curb the inertial urge for perpetual consumption,
and therefore aid in reducing `mindless scrolling' \citep{rauch2018slow}, or even addiction \citep{montag2018internet, ding2016beyond}.
As a general means for promoting well-being,
breaking is psychologically well-grounded \citep[e.g.,][]{danziger2011extraneous,sievertsen2016cognitive}. 

But 
from a learning perspective, breaks are puzzling:
given the extensive efforts systems invest in increasing engagement,
why should they then deliberately propose the opposite? 
Our first goal in this paper is hence to
ground the role of breaks in recommendation.
The key modeling point is that engagement, as a dynamic process, is
driven by two complementing forces:
\emph{positive-feedback effects}, in which useful recommendations reinforce engagement;
and \emph{negative-feedback effects}, in which persistent consumption
gradually exhausts the drive to engage. 
Thus,
systems that seek to promote long-term engagement should learn to correctly balance between these two forces.
Here we advocate for breaks as effective means for this,
enabling prolonged and sustained engagement
by actively preserving user well-being.\squeeze

Towards this, 
we begin by presenting a simple but natural
model of recommendation dynamics 
which incorporates both feedback types.
We then study recommendation in this setting,
and establish when and why breaks make sense.
Intuitively, whereas conventional approaches to recommendation
greedily act to maximize \emph{immediate} engagement,
breaks work to temporarily \emph{reduce} engagement---but with the intent of allowing interest to replenish,
which is helpful in the long run.
Thus, breaks act as a 
balancing force:
if scheduled correctly,
they can sustain, or even increase, long-term engagement.
Our analysis provides conditions for when it is beneficial to break---and when it is not.

Given these insights,
our second goal is to introduce and study the novel task of \emph{learning to suggest a break}.
Current breaking solutions are entirely heuristic, in that they rely on users to manually determine when they should be prompted to break;
thus, they provide no guarantees.
As an alternative, we propose an algorithm for learning optimal breaking policies from data:
for any recommendation policy,
our algorithm finds a breaking schedule that optimizes long-term engagement
by proactively overriding the base policy when this is deemed necessary.
Our approach is practical and efficient,
and enjoys favorable theoretical guarantees.

The challenge in 
suggesting breaks 
is that the effects of recommendations
on users
slowly accumulate over time,
which requires learning to be preemptive.
Towards this, 
and borrowing from the field of population dynamics,
our approach casts recommendation as a dynamical system of Lotka-Volterra (LV) equations \citep{lotka1910contribution}. 
In this space, 
learning to break reduces to solving an optimal control problem,
in which engagement is associated with a certain notion of equilibrium.
We show how to efficiently solve for the optimal equilibrium;
this provides us with a useful breaking schedule that can be utilized in the original problem space.
To the best of our knowledge, the use of LV dynamics to model user behavior in recommendation systems is novel.

Our approach works by embedding users
in `LV-space'---the set of all possible LV trajectories,
which effectively serves us as a parameterized hypothesis class.
Our learning algorithm enjoys the following useful property:
given \emph{predictions} of user engagement,
the policy problem decomposes over users, and can be solved independently for each user.
The final learned policy has a simple interpretation:
it takes as input a small set of predictions, and via careful interpolation, applies a personalized decision rule that anticipates the effects of breaking on future outcomes.

Our main theoretical result is an agnostic bound on the expected long-term engagement of our learned breaking policy, relative to the optimal policy in the class. We show that for any recommendation policy
and any data-generating process,
the optimality gap decomposes into three distinct additive terms: (i) predictive error, (ii) modeling error (i.e., embedding distortion), and (iii) variance around the (theoretical) steady state. These provide an intuitive interpretation of the bound, as well as means to understand the effects of different modeling choices on outcomes. Our proof technique relies on carefully weaving LV equilibrium analysis within conventional concentration bounds for learning.\squeeze

Finally, we provide an empirical evaluation of our approach on semi-synthetic data. Using
two real datasets,
we generate simulated user interaction sequences 
in a way that captures the essence of our model, but is different from the actual continuous-time dynamics we optimize over. Results show that despite this gap, our approach improves significantly over myopic baselines, and often closely matches an optimal oracle. Taken together, these results demonstrate the potential utility of
our approach.
Code is available at:
\LvCodeURL{}.

\subsection{Broader Perspective}
At a high level,
our work argues for viewing recommendation as a task of
\emph{sustainable resource management}.
As other cognitive tasks, engaging with digital content
requires the availability of certain cognitive resources---attentional, executive, or emotional.
These resources are inherently \emph{limited},
and prolonged engagement depletes them
\citep{kahneman1973attention,muraven2000self}.
This, in turn, can reduce the capacity of key cognitive processes
(e.g., perception, attention, 
memory, self-control, and decision-making), 
and in the extreme---cause ego depletion \citep{baumeister1998ego}
or cognitive fatigue \citep{mullette2015cognitive}.
As a means to allow resources to replenish,
`mental breaks' have been shown to be highly effective
\citep{bergum1962vigilance,hennfng1989microbreak,gilboa2008meta,ross2014effects,helton2017rest}.

Nevertheless, traditional approaches to recommendation
remain agnostic to the idea that the recommending in itself takes a cognitive toll on user:
instead, methods simply recommend at each point in time the item predicted to be most engaging
\citep{robertson1977probability}.
As an alternative,
our approach explicitly models recommendation as a process
which draws on cognitive resources---and therefore, must also conserve them.
The subclass of `Predator-Prey' LV dynamics which we rely on are
used extensively in ecology for modeling the 
dynamics of interacting populations,
and demonstrate how over-predation
can ultimately lead to self-extinction
by eliminating the prey population---but also show how enabling resources to naturally replenish ensures sustainable relations.
As such,
here we advocate for studying recommendation systems as
human-centric \emph{ecosystems},
whose sustainability requires active conservation.

\subsection{Related Work}
\label{sec:related}

\paragraph{Recommendation ecosystems.}
Our work pertains to a growing literature that studies recommendation systems as complex systems in which learning plays a distinctive role.
Some works aim to connect micro-level actions
to emerging macro-level phenomena, such as
homogenization via confounding \citep{chaney2018algorithmic},
heterogenization via social learning \citep{schmit2018human},
diversity via strategic behavior \citep{hron2022modeling},
feedback amplification \citep{mansoury2020feedback},
accessibility and stereotyping \citep{guo2021stereotyping},
and the relation between online and offline metrics \citep{krauth2020offline}.
To the best of our knowledge, our work is novel in considering breaks,
but nonetheless, draws tight connections to recent attempts of injecting psychological modeling into recommendation system design \citep{dubey2022pursuit,curemi2022towards}.
Due to the counterfactual nature of recommendation,
most works in this field provide either theoretical analysis,
or simulation studies \citep{ie2019recsim,chaney2021recommendation}.
We follow suit, and aim for both.\squeeze

\paragraph{User dynamics and feedback.}
Several recent works seek to capture time-varying user behavior
by modeling users as acting based on dynamic latent states.
These differ from ours in two important ways.
First,
most works consider either positive-only feedback \citep{kalimeris2021preference,passino2021where,dean2022preference}, 
or negative-only feedback
\citep{wang2003personalization,warlop2018fighting,kleinberg2018recharging, cao2020fatigue, leqi2021rebounding}.
We consider both types of feedback, and how they interact,
which we believe is more realistic---as well as necessary for explaining breaks. 
Second,
most works consider discrete time with fixed intervals,
and are therefore incapable of modeling variation in (continuous) consumption rates, which is our primary focus.
One exception is \citet{du2015time},
who study continuous-time Hawkes temporal point processes;
but since these can only express excitation effects, 
they are inherently restricted to positive-only feedback.
Perhaps closest to ours,
\citet{kleinberg2022challenge},
also model bi-directional feedback,
but in a very different setup
(length of a single session)
and towards different aims 
(characterizing equilibira, rather than learning). 

\paragraph{Lotka-Volterra dynamics.} 
The study of ecosystem dynamics and their conservation has a long and rich history, in which LV analysis is integral \citep[see][]{hofbauer1998evolutionary,takeuchi1996global}.
LV systems are used primarily for modeling biological ecosystems,
but are also used in economics \citep{weibull1997evolutionary,samuelson1998evolutionary},
finance \citep{farmer2002market,doi:10.1073/pnas.2015574118},
and behavioral modeling \citep[e.g.,][modeling drug addiction and relapse]{duncan2019fast}.
We believe our work is novel in its use of LV modeling in recommendation.
In terms of learning, \citet{gorbach2017scalable} and \citet{ryder2018black} propose variational techniques for dynamical systems, but do not consider control.
Our work aims to directly learn optimal policies,
drawing on recent advances in turnpike optimal control
\citep{trelat2015turnpike}.\squeeze

\section{Problem Setting}
\label{section:setting}

We consider a sequential recommendation setting
in which users interact with a stream of recommended items over time.
New users $u \in \Users$ are sampled iid from some unknown distribution $D$,
and begin interacting with the system.
In each interaction, the system recommends an item $x$ from the set of available items $\Items$,
and users respond by reporting as feedback their rating $r \in \R$ for $x$.
We assume recommendations are governed by an existing and fixed \emph{recommendation policy} $\pi_0$,
which we refer to as the `base' policy.
For concreteness, we follow \citet{dean2020designing, kalimeris2021preference, hron2022modeling} and model recommendations as being made on the basis of a personalized score function 
$\hat{r}(u,x)$, 
trained to estimate $\hat{r}(u,x)\approx r$,
such that each item $x$ is recommended to $u$  
via the softmax rule:\squeeze
\begin{equation}
\prob{\pi_0}{x \mid u} \propto e^{\frac{1}{\mathcal{T}}\hat{r}(u,x)}
\end{equation}
We note however that our approach supports any $\pi_0$.

\paragraph{Engagement.}
The overall goal of the system is to maximize engagement,
which we define as the number of interactions until some
chosen time horizon $T$.
Setting $t=0$ as the (relative) time of joining for each user,
the \emph{interaction sequence} of user $u$
under a recommendation policy $\pi$ is:
\begin{equation}
\label{eq:interaction_sequence}
S_u = \Set{(t_i, x_i, r_i)\mid t_i \le T}
\sim \TPP(\pi ; u)
\end{equation}
where $t_i \in \R_+$ is the time of the $i$-th event,
$x_i$ is the recommended item, 
and $r_i$ is the reported rating.
$\TPP(\pi;u)$ is some unknown distribution over sequences,
which permits dependence between tuples $(t_i,x_i,r_i)$ over time.\footnote{Formally, $\TPP$ is a temporal point process (TPP) with markings. 
}
Defining $\tfrac{1}{T}\Size{S_u}$ as the \emph{engagement rate} of $u$,
the system seeks to maximize:\squeeze
\begin{equation}
    \label{eq:long_term_engagement_rate}
    \LongTermEngagement(\pi) = 
    \dblexpect{u \sim D}
    {S_u\sim\TPP(\pi;u)}
    {\tfrac{1}{T}\Size{S_u}}
\end{equation}
which we refer to as \emph{expected long-term engagement rate}.

\paragraph{Breaking policies.}
We are interested in understanding how breaks affect engagement
when applied on top of an existing recommendation policy.
Formally, we consider compound policies $\pi = \psi \circ \pi_0$,
where $\pi_0$ is the (fixed) base policy,
and $\psi \mapsto \{0,1\}$ is a learned \emph{breaking policy} that can either override the base policy by prompting the user to break ($\psi=1$),
or pass on the recommended item $x \sim \pi_0$ ($\psi=0$).
Thus, our learning objective is:
\begin{equation}
    \argmax\nolimits_{\psi \in \Psi} \,\LongTermEngagement(\psi \circ \pi_0) 
\end{equation}
where $\Psi$ is a class of individualized breaking policies.
For simplicity, we assume that $\pi_0$ does not incorporate breaks.

\begin{figure}
    \centering
    \includegraphics[width=0.9\columnwidth]{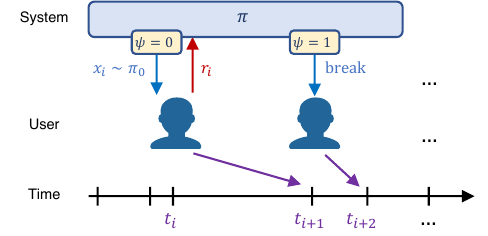}
    \caption{
    Interaction between the system and a user $u$ over time. 
    Once $u$ decides to query the system at time $t_i{\in}\R$,
    the policy $\pi$ decides whether to recommend an item $x_i\sim\pi_0$ ($\psi=0$), or suggest a break ($\psi=1$).
    In response, the user provides rating feedback $r_i$ on content,
    and decides when to interact next ($t_{i+1}$). 
    }
    \label{fig:interaction_model}
\end{figure}

\subsection{Engagement Dynamics} \label{sec:engagement_dynamics}
To optimize engagement, we must be precise about what determines $S_u$.
Broadly, we think of $S_u$ as constructed sequentially by the user, where at each time $t_i$, and based on her experience with the recommended $x_i \sim \pi$,
the user decides on her next time of interaction $t_{i+1}$.
The process is illustrated in \cref{fig:interaction_model}.
In what follows, we discuss dynamics focusing on a single user $u$,
and hence for clarity drop notational dependence.
We return to discussing  multiple users in \cref{sec:learning}.

Our first step to defining $S_u$ considers rates.
Notice that maximizing the number of events $|S_u|$ is analogous to
minimizing gaps between consecutive user queries, $\Delta t_i = t_{i+1} - t_i$;
this, in turn, is akin to maximizing \emph{instantaneous rates}:\footnote{
As $\tfrac{1}{T}\Size{S_u}$ is a rate of events with instantaneous frequency $\lambda_i$, maximizing the empirical rate 
is asymptotically
equivalent to maximizing the \emph{harmonic mean} of $\lambda_i$. 
See \cref{section:empirical_rate_harmonic_mean}.
}
\begin{equation}
\label{eq:instantaneous_rate}
\lambda_i = \Delta t_i^{-1}
\end{equation}
We will henceforth use rates to discuss engagement,
which will prove useful when considering limiting behavior.

Our next step is to associate $\lambda_i$ with recommendations.
Naturally, we expect good recommendations to entail frequent revisits; hence, we begin by setting:
\begin{equation}
\label{eq:stateless_model}
\lambda_i = \beta_i \triangleq \beta(r_i)
\end{equation}
where $\beta(\cdot)$ is some latent mapping from ratings $r_i$,
interpreted here as the utility for $u$ from consuming $x_i$,
to temporal behavior.
We imagine $\beta(\cdot)$ as generally being monotonically increasing in $r_i$,
so that more relevant content triggers more frequent visits to the platform.
\eqref{eq:stateless_model} is useful---but is limited in that 
it models the $\Delta t_i$ as temporally independent;
by this, it restricts any variation in engagement to stemming exclusively from recommended items---and not from users.

To remedy this, our next step is to make distinctive the role of users in the process.
In particular, we propose to introduce \emph{momentum}---by considering the $\lambda_i$ as latent user states,
and allowing $\lambda_i$ to depend on the previous $\lambda_{i-1}$ as:
\begin{equation}
    \label{eq:momentum_model}
    \lambda_i = \lambda_{i-1} \left(1-\alpha+\beta_i \right)
\end{equation}
for some constant $0 < \alpha < 1+\beta_i$.
\Eqref{eq:momentum_model} 
asserts that engagement rate $\lambda_i$ increases if the utility $\beta_i$ is larger than some natural decay parameter $\alpha$, and decreases otherwise;
if $\alpha=\beta_i$, then the rate is constant.
When recommendation quality is low and $\beta_i$ regularly fails to exceed $\alpha$,
engagement rate drops to zero, and users leave the system.
But the converse setting---in which recommendations are effective
and $\beta_i$ is always larger than $\alpha$---implies that engagement \emph{increases indefinitely}, 
which is unrealistic.

What is missing in \eqref{eq:momentum_model} is a balancing force that regulates consumption.
Our final step is therefore to introduce an additional latent variable, $z_i \in [0,1]$, which we think of as `interest',
and whose role is to stabilize consumption via:\footnote{Similar notions have been considered in \citet{leqi2021rebounding} who model satiation, and \citet{kleinberg2018recharging} who model fatigue; these, too, model variation in affinity, rather than time.}
\begin{equation}
    \label{eq:discrete_lv_lambda}
    \lambda_i = \lambda_{i-1} \left(1-\alpha+\beta_i z_i \right)
\end{equation}
Thus, the positive effect of $\beta_i$ on $\lambda_i$ is mediated
by current interest $z_i$.
Our key modeling point is that interest should deplete with extended consumption;
thus, we model $z_i$ as also varying with time,
and w.r.t. $\lambda_{i-1}$, as: 
\newcommand{\DiscreteLVeqref}{Eqs.~(\ref{eq:discrete_lv_lambda}, \ref{eq:discrete_lv_z})}
\begin{equation}
    \label{eq:discrete_lv_z}
    z_i =  z_{i-1} \left(1+\gamma(1-z_{i-1}) - \delta \lambda_{i-1} \right)
\end{equation}

for some constants $\gamma,\delta>0$.
Note that \DiscreteLVeqref\
are functionally similar,
differing only in the sign of the constants
(as per their opposing roles),
and in the term $(1-z_{i-1})$ which ensures that $z_i$ remains in $[0,1]$.
We refer to \eqref{eq:discrete_lv_lambda} as \emph{positive feedback} and to \eqref{eq:discrete_lv_z} as \emph{negative feedback}.

\subsection{Optimizing Suggested Breaks}
Recall our goal is to learn an optimal breaking policy $\psi$.
Breaks are expressed in the dynamic model by setting
$\beta=0$ and $\delta=0$: 
this causes $\lambda_i$ to temporarily decrease (due to $-\alpha$),
but allows $z_i$ to replenish (thanks to $+\gamma$).
To learn effective breaking schedules,
our general strategy will be to optimize for `sustainable habits',
which we think of as the limiting behavior of $\tfrac{1}{T}\Size{S_u}$
as $T{\rightarrow}\infty$.
Thus, habits refer to consumption behaviors that are stable on average, but nonetheless may exhibit some variability over time.
To understand the effects of breaking on limiting behavior,
we now switch to discussing dynamics in continuous-time.

\section{Engagement in Continuous Time}
\label{section:continuous_lv}

One challenge in optimizing \eqref{eq:long_term_engagement_rate} is that empirical rates $\tfrac{1}{T}\Size{S_u}$ exhibit variation that may be difficult to account for using observed data.
As an alternative, we will aim for optimizing 
individualized \emph{limiting rates}, defined as:
\begin{equation}
\label{eq:limiting_rate}
\lambda^*_u = \lim\nolimits_{T\to\infty} \tfrac{1}{T}\Size{S_u}
\end{equation}
This abstracts away `everyday' variation in behavior,
and focuses instead on habits---which are easier to anticipate.
When empirical rates are `well behaved' in the sense that
they concentrate around the limiting behavior,
then we can expect $\lambda^*_u$ to be a good proxy for engagement.
In \cref{sec:theoretical_analysis} we make this notion precise.
To understand the possible effects of breaks on limiting behavior,
we will analyze a continuous-time analog of our dynamic model  in Sec.~\ref{sec:engagement_dynamics}.
This will allow us to employ powerful tools from dynamical systems and control theory,
and establish preliminary results that will form the basis of our learning approach.\squeeze

\begin{figure}
    \centering
    \includegraphics[width=\columnwidth]{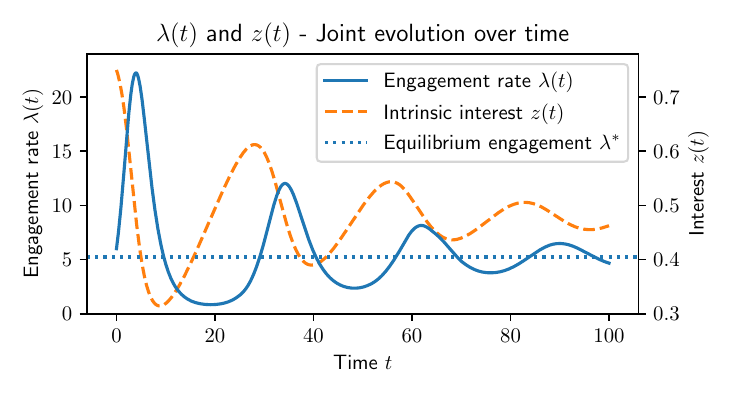}
    \figvspace{}
    \caption{
    Temporal relations between \emph{rate} $\lambda(t)$,
    \emph{interest} $z(t)$, and equilibrium $\lambda^*$ in the continuous limit (\eqref{eq:lv}).
    Note how $\lambda(t)$ drops only some time after $z$ has depleted.
    }
    \label{fig:consumption_cycle_dynamics}
\end{figure}

\subsection{Continuous Engagement Dynamics}
Mapping \DiscreteLVeqref{} to continuous time requires three steps.
First, we define $\lambda(t)$ and $z(t)$
as the continuous analogs of $\lambda_i$ and $z_i$.
Next, we overload notation and define
$\beta=\expect{\pi}{\beta_i}$,
to be the expected value of the random variable $\beta_i$ (Eqs. (\ref{eq:stateless_model}-\ref{eq:discrete_lv_lambda})), interpreted as the `average effect' of discrete recommendations on behavior.
Lastly, we account for breaks:
Consider some breaking policy $\psi$, 
and denote by $p$ the expected breaking rate, namely $p=\prob{}{\psi=1}$.
Under $\psi$, the expected values of $\beta$ and $\delta$
become $(1-p)\beta$ and $(1-p)\delta$, respectively.
This gives our final continuous model:\footnote{
\DiscreteLVeqref{} can be obtained from \Eqref{eq:lv} via the Euler method.
In the case of stochastic recommendations, the dynamics are not smooth, but cumulative rates of engagement still tend to converge towards the LV equilibrium (see \cref{fig:discrete_and_continuous_dynamics}).
}\squeeze
\begin{align}
\label{eq:lv}
\begin{split}
\tFullDerivative{\lambda}{t}
&= -\alpha \lambda + \beta z \lambda (1-p)
\\
\tFullDerivative{z}{t}
&= \gamma z (1-z) - \delta z \lambda (1-p)
\end{split}
\end{align}
\Eqref{eq:lv} describes a system of Lokta-Volterra (LV) differential equations,
characterized by the set of parameters
$\theta=(\alpha,\beta,\gamma,\delta)$,
and $p$.
LV systems,
also known as \emph{predator-prey dynamics},
have been popularized by and studied extensively in the fields of theoretical ecology and population dynamics \citep{hofbauer1998evolutionary, takeuchi1996global}.
We now proceed to overview some useful properties of LV systems.\squeeze

\paragraph{Cycling behavior.}
\Eqref{eq:lv} describes user behavior as a \emph{cycle}:
when interest $z(t)$ is high, engagement rate $\lambda(t)$ increases,
resulting in positive feedback;
conversely, when $\lambda(t)$ is high, $z(t)$ decreases,
which expresses negative feedback.
In general,
$\lambda$ grows until interest is too low to
sustain consumption, at which point consumption drops sharply,
allowing interest to recover---and the cycle repeats.
The cycling behavior exhibits oscillations in $\lambda$ and $z$,
with one lagging after the other.
A typical trajectory is illustrated in 
\cref{fig:consumption_cycle_dynamics}.
Note how the drop in $\lambda$ occurs only some time after $z$ is depleted;
hence, anticipating (and preventing) the collapse of $\lambda$ requires 
conservation
of $z$.
Thus, $z$ serves as a resource:
necessary for engagement,
and of limited supply.

\paragraph{Equilibrium.}
Over time,
and if no interventions are applied,
the magnitude of oscillations decreases,
and the system naturally approaches a \emph{stable equilibrium}, denoted $(\lambda^*_\theta, z^*_\theta)$,
determined by system parameters $\theta$,
and which attracts all initial conditions $\lambda(0),z(0)>0$ \citep{takeuchi1996global}.
Our first result shows that $\lambda^*_\theta$ has a convenient closed form.

\newpage
\begin{lemma}
\label{lemma:control_equilibrium}
Let $\theta=(\alpha,\beta,\gamma,\delta)$ define a controlled LV system as in \eqref{eq:lv}.
Then for any $p \in [0,1]$, we have:
\begin{equation}
\label{eq:control_equilibrium}
\lambda^*_\theta(p) =
\max\Set{
\frac{\gamma}{\delta}\frac{1}{1-p}\left(1-\frac{\alpha}{\beta}\frac{1}{1-p}\right)
,0
}
\end{equation}
\end{lemma}
Proof in \cref{section:single_channel_lv_properties},
and illustration in \cref{fig:lv_equilibrium} (Left).
\Eqref{eq:control_equilibrium} is useful as it depicts
$\lambda^*$ as a simple function of $p$, parameterized by $\theta$.
This will prove useful for optimization. 

\subsection{Breaking as Optimal Control}
\Eqref{eq:lv} suggests a natural approach for optimizing breaks:
given $\theta$,
find $p$ that maximizes the limiting rate $\lambda^*$.
Essentially, this casts $p$ as a \emph{control variable},
and learning to break becomes a problem of optimal control.
Note how $p$ 
mediates the relations between $\lambda(t)$ and $z(t)$:
when $p>0$, 
it decelerates engagement rate $\lambda$,
and at the same time, lets $z$ recover.\squeeze

Our goal is now to solve the optimal control problem:
\begin{equation}
\label{eq:optimal_control}
p^*(\theta) = \argmax\nolimits_{p \in [0,1]} \lambda^*_{\theta}(p)    
\end{equation}
Towards, this,
our next result derives a closed form solution for the optimal $p^*$.
Note that \eqref{eq:control_equilibrium}
shows  $\lambda^*_\theta(p)$ is piece-wise polynomial in $q=1/(1-p)$.
Solving for $q$, we get: 
\begin{lemma}
\label{lemma:optimal_p_informal}
Let $\theta=(\alpha,\beta,\gamma,\delta)$ define an LV system as in \eqref{eq:lv}.
Then the optimal $p^*$ is given by:
\begin{equation}
\label{eq:optimal_p}
p^*({\theta}) =
\max \Set{
1-2\frac{\alpha}{\beta} 
,0
}
\end{equation}
\end{lemma}
Proof in \cref{section:single_channel_lv_properties}.
See illustration in Fig.~\ref{fig:lv_equilibrium} (Right).

\begin{figure}[t]
    \centering
    \includegraphics[width=\columnwidth]{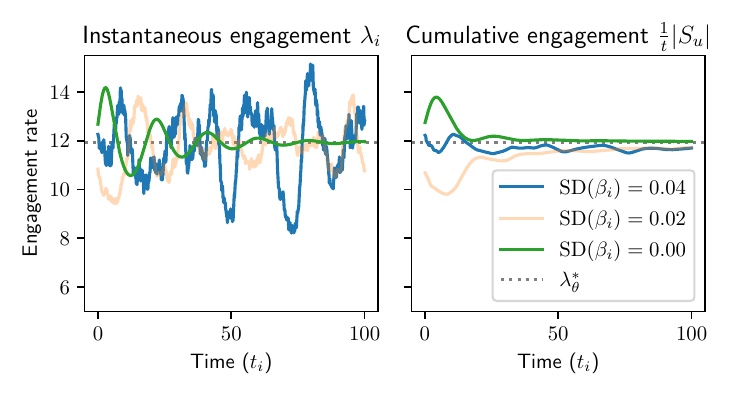}
    \figvspace{}
    \caption{
    Discrete engagement dynamics for varying levels of dispersion in the distribution of $\beta_i$. 
    \textbf{(Left)} When the variance of $\beta_i$ is low, instantaneous engagement $\lambda_i$ approaches smooth LV dynamics. 
    \textbf{(Right)} Cumulative engagement $\tfrac{1}{T}\Size{S_u}$ converges towards the 
    equilibrium $\lambda^*_\theta$ (\eqref{eq:control_equilibrium}) even under strong stochasticity.
    }
    \label{fig:discrete_and_continuous_dynamics}
\end{figure}

\subsection{Learning to Break, Revisited}
In continuous space, breaking manifests as a control variable $p \in [0,1]$, continuous and fixed in time.
To apply this idea back in our original discrete problem setting,
we can interpret $p$ as determining the probability to break on any given input.
This defines a class of stationary breaking policies:\squeeze
\begin{equation}
\label{eq:stationary_policy}
\Psi = \left\{\psi_u(p) = \text{break w.p. $p$} \,:\, p \in [0,1], u \in \Users \right\}
\end{equation}
Using this, our general approach for learning to break will be to:
(i) associate with each user $u$ a set of LV parameters $\theta_u$, and then
(ii) compute $p_u^*$ which maximizes $\lambda^*_{\theta_u}(p)$,
and apply breaks using the individualized policy $\psi_u = \psi_u(p_u^*)$.
In principal, we expect $\lambda^*_{\theta_u}$
to be effective as a proxy for engagement when the
observed $\frac{1}{T}|S_u|$ express empirical habits that are
`close enough' to the theoretical equilibrium.
In \cref{sec:theoretical_analysis} we make this notion precise.

One useful property of optimal LV policies is that they suggest breaks only when this is deemed necessary.
Note that by \eqref{eq:optimal_p},
$p^*(\theta)$ exhibits a \emph{phase transition}
at $\tfrac{\alpha}{\beta}{=}\frac{1}{2}$, below which $p^*{>}0$, and above which $p^*{=}0$.
When considering individualized $\theta_u$,
we get the following result:
\begin{corollary}
\label{cor:phase_shift}
In LV space, 
users are partitioned by their $\theta_u$
to those who benefit from breaks, and those who don't.
\end{corollary}
We further explore this idea
empirically in \cref{sec:experiments}.

\section{Learning Optimal Breaking Policies}
\label{sec:learning}

\begin{figure}
    \centering
    \includegraphics[width=\columnwidth]{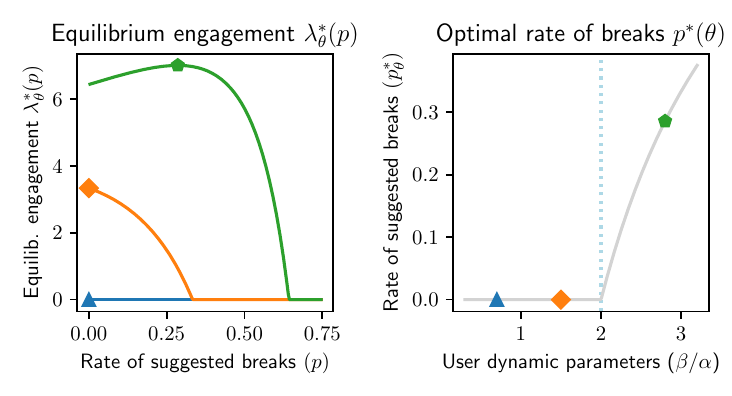}
    \figvspace{}
    \caption{
    \textbf{(Left)} Equilibrium curves $\lambda^*(p)$
    and optimal policies $p^*$ (\emph{markers})
    for user types ($\theta_u$) that:
    benefit from breaks (\emph{green}),
    do not require breaks (\emph{orange}),
    and will inevitably churn (\emph{blue}). 
    Curves are given by \cref{lemma:control_equilibrium}.
    \textbf{(Right)} The optimal policy $p^*$ for all $\beta/\alpha$, as given by \cref{lemma:optimal_p_informal}.
    The optimal policy exhibits a 
    second-order phase change at ${\beta}/{\alpha}=2$ (see \cref{cor:phase_shift}).
    }
    \label{fig:lv_equilibrium}
\end{figure}

We now turn to presenting our learning algorithm.
As noted, the algorithm consists of two steps:
(i) \emph{embedding}, which fits for each user some $\hat{\theta}_u$ from data,
and (ii) \emph{optimization}, which computes an optimal $\hat{p}_u$
from $\hat{\theta}_u$.
One benefit of our approach is that it operates entirely on predictions of future engagement.
Our procedure is illustrated in  \cref{fig:policy_optimization_from_data}.

\subsection{Embedding Users in LV Space}
\label{sec:embeddings}
Our first task is to choose a suitable $\theta_u$ for $u$,
which we think of as embedding users in `LV space'.
A natural first attempt would be to fit $\theta_u$ to $\lambda^*_\theta$ in \eqref{eq:control_equilibrium} from data.
However, the crux is that the observed $S_u$ come from the default policy $\pi_0$, which does not include breaks,
i.e., has $p=0$.
But our ultimate goal is to optimize over \emph{all} $p$---for this,
data that is representative of a single $p$ (e.g., $p=0$)
will likely be biased.\squeeze

\paragraph{Equilibrium curves.}
Ideally, what we would like to do is fit $\theta_u$ to the
entire \emph{equilibrium curve} of $\lambda^*_\theta(p)$.
Let $\bar{\lambda}_u(p)$ be the \emph{expected empirical engagement rate}, defined as:
\begin{equation}
\label{eq:expected_empirical_rate}
\bar{\lambda}_u(p) = \expect{S_u\sim\TPP(\psi(p);u)}{\tfrac{1}{T}\Size{S_u}}
\end{equation}
As a function of $p$,
$\bar{\lambda}(p)$ gives the \emph{true} engagement rate
for any choice of $p$. 
Using this notation, we seek $\theta$ for which 
$\lambda_\theta^*(p)$ closely aligns with that of $\bar{\lambda}_u(p)$
across all $p \in [0,1]$:
\begin{equation}
\label{eq:theta_ideal_fit}    
\bar{\theta}_u = 
\argmin\nolimits_{\theta} \| \bar{\lambda}_u-\lambda^*_\theta \|
\end{equation}
for some function norm $\|\cdot\|$,
and for which $\lambda^*_{\bar{\theta}_u}$
and $\bar{\lambda}_u$ have similar maximizing $p$.
Unfortunately, $\bar{\lambda}_u$ is a counterfactual object, 
as solving \eqref{eq:theta_ideal_fit} requires knowledge of $\bar{\lambda}(p)$ for all $p\in[0,1]$, whereas our original data offers just one for each user.
To overcome this obstacle, we make use of \emph{predictions} obtained through experimentation and supervised learning.

\begin{figure}
    \centering
    \includegraphics[width=\columnwidth]{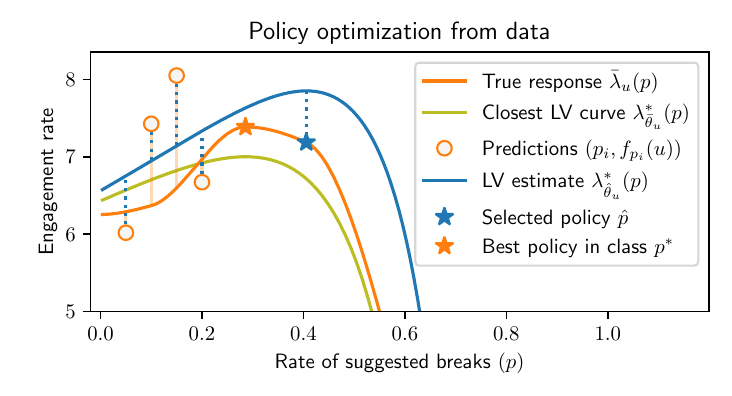}
    \figvspace{}
    \caption{
    Schematic diagram of the learning method described in \cref{sec:learning}.
    An illustration of
    the true counterfactual engagement curve (\emph{orange}),
    an ideal LV fit (\emph{tan}),
    and an empirical LV fit (\emph{blue}) from observations (\emph{circles}), showing similar optima (\emph{stars}).
    }
    \label{fig:policy_optimization_from_data}
\end{figure}

\paragraph{The role of prediction.}
As any policy problem, learning to break requires some form of exploration or experimentation.
Here we aim for experimentation to be simple and minimal.
Specifically,
we will allow the system to collect some additional data:
for a small set of $N$ distinct breaking rates $p_j{>}0$,
we assume the system can allocate some bandwidth
to experiment using compound policies
$\pi_j = \pi(p_j)=\psi(p_j)\circ\pi_0$, and obtain data $D^{(j)} = \{ (u_k,S_{u_k}) \}_{k=1}^{m_j}$
for small $m_j$.

Denoting by $D^{(0)}$ the 
original data for the base policy $\psi(p=0)\circ\pi_0$,
we
use the gathered 
$D^{(0)},\dots,D^{(N)}$
to learn individualized \emph{policy-specific predictors},
$f_j(u) = f_{p_j}(u)$, trained to predict for each user $u$ her engagement rate
$y=\frac{1}{T}|S_u|$ under the policy $\pi_j$.
For example, if we train $f_j(u)$ to minimize the squared error
$\sum_k \left(f_j(u_k) - \tfrac{1}{T}\Size{S_{u_k}} \right)^2$ on pairs $(u_k,S_{u_k}) \in D^{(j)}$,
then $f_j(u)$ should be a reasonable estimator of the expected $\bar{\lambda}_u(p_j)$.
Hence,
for a given $u$,
a finite set of pairs $\{(p_j,f_j(u))\}_{j=1}^N$ 
gives points
to which we can fit $\theta$ to $\lambda^*_\theta$.
Our final criterion for choosing $\hat{\theta}_u$ is:\looseness=-1
\begin{align}
\label{eq:theta_empirical_fit}
\begin{split}
\hat{\theta}_u 
&= \argmin\nolimits_{\theta} \sum\nolimits_{j=1}^N \left( f_j(u) - \lambda^*_\theta(p_j) \right)^2
\end{split}
\end{align}
where $\f(u) = (f_1(u), \dots, f_N(u)) \in \R^N_{>0}$, 
and given here with the $\ell_2$ vector norm.
From \cref{lemma:optimal_p_informal},
optimizing over $p$ requires only the ratios
$\alpha_u/\beta_u$ and $\gamma_u/\delta_u$, which appear as polynomial coefficients.
Hence,
\eqref{eq:theta_empirical_fit} can be efficiently solved
using a polynomial Non-Negative Least Squares (NNLS) regression solver \citep{chen2010nonnegativity}.

\paragraph{The role of experimentation.}
In the realizable case, 
\eqref{eq:theta_ideal_fit} has a zero-norm minimizer, 
and the goodness of fit for $\hat{\theta}_u$ is controlled by two parameters:
the number of experimental datasets, $N$,
and their sizes, $m_j$ for $j \in [N]$.
In general, increasing $N$ provides more data points for solving \eqref{eq:theta_empirical_fit}, and increasing each $m_j$ reduces noise for that point
(i.e., $f_u(p)$ should be closer to $\bar{\lambda}_u$).
However, in reality experimentation is costly, and so $N$ and the $m_j$ may be small.
As motivation, we next show that under realizability
and for accurate predictions,
$N=1$ suffices. 
Our result applies to more general base policies $\pi_0=\pi(p_0)$
using any $p_0 \ge 0$.\squeeze
\begin{proposition}
\label{prop:realizable_closed_form}
Fix $N=1$, and let $p_0, p_1 \in [0,1-\alpha/\beta]$.
For a user $u$, if
(i) exists $\theta_u$ s.t. $\bar{\lambda}_u = \lambda^*_{\theta_u}$,
and (ii) $f_i(u) = \bar{\lambda}_u(p_i)$ for $i = 1,2$,
then solving \eqref{eq:theta_empirical_fit}
recovers the true expected rate, i.e., $\hat{\theta}_u = \bar{\theta}_u$, and is therefore optimal.
\end{proposition}
Proof provided in \cref{section:single_channel_lv_properties}, and relies on \cref{lemma:control_equilibrium}.
Next, we discuss how to obtain $\psi_u$ from $\theta_u$.

\begin{figure*}[th!]
    \centering
    \includegraphics[width=\linewidth]{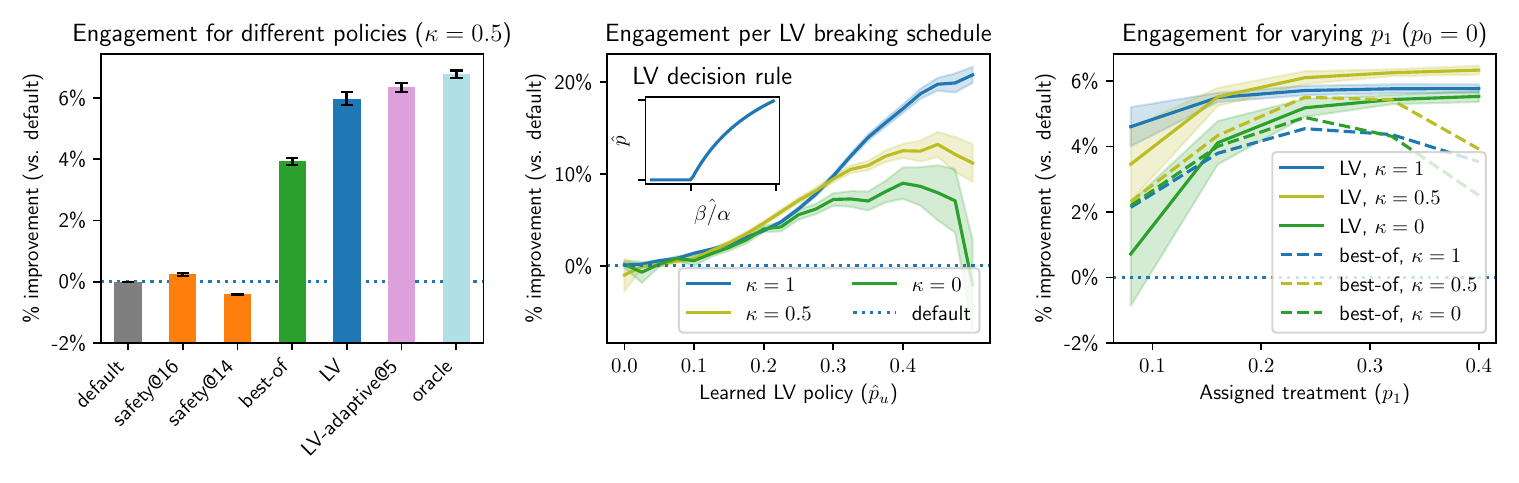}
    \figvspace{}
    \caption{
    Results on the MovieLens 1M dataset.
    \textbf{(Left)}
    Performance gain of different approaches
    (relative to \method{default} policy).
    \textbf{(Center)}
    Performance of \method{LV} by user group,
    partitioned by 
    learned policies $\hat{p}_u$.
    \textbf{(Right)}
    Sensitivity to 
    increasing
    experimental $p_1$
    ($N=2, p_0=0$).
    }
    \label{fig:experimental_results}
\end{figure*}

\subsection{From Predictions to Optimal Policies}
\label{sec:optimal_policies}

One useful property of our approach is that it circumvents 
the need to learn a global policy:
Once the $\{f_j(u)\}_j$ have been learned,
the policy problem decomposes over users,
and optimal individualized policies $\psi_u$
are determined independently for each user.
I.e., by relying on predictions,
the solution to \eqref{eq:long_term_engagement_rate}
is immediately obtained,
and at test time we simply use predictions to compute $\psi_u$
for new users $u$.

Our final procedure is as follows:
given some user $u$, we (i) compute predictions $\f(u)$;
(ii) find $\hat{\theta}_u$ by solving \eqref{eq:theta_empirical_fit};
(iii) obtain $\hat{p}_u$ by solving \eqref{eq:optimal_p};
and (iv) apply the policy:
\begin{equation}
\label{eq:learned_phat}
\psi_u =\psi(\hat{p}_u),
\,\,\quad \text{where} \,\,\,
\hat{p}_u = p^*(\hat{\theta}_u)
\end{equation}

Notably, for $N=1$,
$\psi_u$ has a closed-form formulation 
as a function of predictions
(\Appendixref{section:appendix_model_fit}).
In this case:
\begin{corollary} 
\label{cor:phase_shift2}
In the realizable case of \cref{prop:realizable_closed_form},
$\psi(\hat{p}_u)$ idempotently improves over the myopic $\pi(0)$. 
\end{corollary}

Thus, the optimal policy can be interpreted as suggesting breaks only when it deems them necessary. \cref{fig:lv_equilibrium} illustrates 
$\lambda^*(p)$ curves and policies for various user types.

\subsection{Beyond Stationary Policies}
\label{subsec:beyond_stationary}
Stationary policies (\eqref{eq:stationary_policy})
enable efficient policy optimization 
because they allow us to use predictions
as proxies for counterfactuals.
However, as the predictions they rely on are fixed, stationary policies cannot make use of new data to improve or to adapt to changes.
If such new data becomes available as the learned policy is being deployed,
then a simple procedure for constructing a non-stationary \emph{adaptive policy}
can
periodically update either the learned predictive models (as done in retraining)
or the inputs they rely on,
and then re-solve
Eqs.~(\ref{eq:theta_empirical_fit}) and (\ref{eq:optimal_p})
to obtain the updated $\hat{p}$.
In \cref{sec:experiments}, we empirically investigate
one implementation of this approach,
but leave the broader exploration of more general non-stationary breaking policies to future work.

\section{Theoretical Guarantees}
\label{sec:theoretical_analysis}
Our main theoretical result
bounds the expected long-term engagement obtained by our global learned policy, $\hat{\psi}=\psi(\hat{p})$.
Our bound shows that the gap between $\hat{\psi}$ and the optimal 
stationary policy $\psi^{\mathrm{opt}}$
is governed by three additive terms, each relating to a different aspect of our approach:
modeling error ($\varepsilon_\textsc{lv}$),
predictive error ($\varepsilon_\mathrm{pred}$),
and deviation from expected behavior ($\varepsilon_\mathrm{dev}$).
A description and interpretation of each term follows shortly.
Note the bound is agnostic, i.e.,
holds for \emph{any} temporal data-generating process,
and thus extends beyond the model presented in Sec.~\ref{sec:engagement_dynamics}.
For simplicity, we focus on $N=1$, 
and again allow for general $\pi_0$.\squeeze
\begin{theorem}[Informal]
\label{thm:bound}
For any $\pi_0$, let $p_0,p_1 \in [0,1]$, and denote by 
$\psi^{\mathrm{opt}} \in \Psi$ be the optimal stationary policy.
Then for the learned breaking policy $\hat{\psi}$, we have:
\begin{equation*}
\LongTermEngagement(\psi^{\mathrm{opt}})
- \LongTermEngagement(\hat{\psi})
\le 
\frac{\eta_{\textsc{tpp}}}{\Abs{p_1-p_0}}
(\varepsilon_\textsc{lv}+\varepsilon_\mathrm{pred}+\varepsilon_\mathrm{dev})
\end{equation*}
where $\eta_{\textsc{tpp}}$ is an $\TPP$-specific constant scale factor.
\end{theorem}
Formal statement, precise definitions, and proof are given in \Appendixref{subsection:single_channel_optimal_policy_proofs}.
The proof consists of three main steps:
Starting with a clean LV system at $T = \infty$, we quantify
the downstream effects of perturbing the optimal policy.
Then, we plug in the learned policy, and 
bound the gap due to
predictive errors and finite $T$.
The final step makes the transition from continuous dynamics to our discrete dynamic model.

We now proceed to detail the role of each of the terms in the bound, and how they may be controlled.

\textbf{Predictive error}:
Since targets $y=\frac{1}{T}|S_u|$ are 
predicted,
$\varepsilon_\mathrm{pred}$ is simply the expected regression error over users,
measured 
in RMSE. As is standard, $\varepsilon_\mathrm{pred}$ can be reduced 
by increasing the number of samples $m$, 
or by learning more expressive 
predictors
$f$ (e.g., larger neural nets).

\textbf{Modeling error}:
LV dynamics permit tractable learning;
but as any hypothesis class,
this trades off with model capacity.
Here, $\varepsilon_\textsc{lv}$ quantifies the error due
limited expressive power.
Further reducing $\varepsilon_\textsc{lv}$ can be achieved by considering richer dynamic models---a challenge left for future work.

\textbf{Deviation from expectation}:
The learned $\hat{p}_u$ rely on predicted equilibrium,
but are trained on finite-horizon data. 
In expectation, 
$\varepsilon_\mathrm{dev}$ captures how finite sequences deviate from their mean.
As a rule of thumb, we expect larger $T$ to reduce this form of noise,
but this cannot be guaranteed.\looseness=-1

\textbf{Sensitivity} :
For $N=2$, the term $|p_1-p_0|$ quantifies the added value of exploring
beyond the default breaking policy of $p_0$.
Intuitively, when the points are farther away, fitting the equilibrium curve is easier.
Thus, for 
$p_0=0$,
$p_1$ should be chosen to balance between performance gain and overexposure of experimental subjects to breaks.

\section{Experiments}
\label{sec:experiments}

We conclude with an empirical evaluation of our approach on semi-synthetic data.
We experiment with two real datasets:
MovieLens 1M, which we present in depth here;
and Goodreads, which exhibits similar results,
and is therefore deferred to \cref{subsec:goodreads_analysis}.
Further details on setup, data generation, and optimization
can be found in \cref{app:experimental_details}.
\cref{app:additional_empirical_evaluation} contains additional experimental results.

\subsection{Experimental Setup}
\label{subsec:experimental_setup}
\paragraph{Data.}
The MovieLens 1M dataset \citep{harper2015movielens}
includes \LvExperimentCfNRatings{} ratings provided by \LvExperimentCfNUsers{} users and for \LvExperimentCfNItems{} items,
which we use to obtain features, determine the dynamics, and emulate $\pi_0$.
We sample and hold out 
\LvExperimentCfTrainingSetPct{}\% of all ratings $r_{ux}$ via user-stratified sampling,
to which we apply Collaborative Filtering (CF)
to get user features $u$ and item features $x$
that approximate $u^\top x \approx r_{ux}$
($d=\LvExperimentCfNFactors$, $\mathrm{RMSE}=\LvExperimentCfRmse$, $r \in [1,5]$). 
This mimics a process where 
long-term predictions are made according to an initial set of items.
We then take the remaining data points and randomly assign \LvExperimentEvaluationSetUsers{} users to the test set, on which we evaluate policies.
The remaining users are randomly assigned to the train sets $\{D^{(j)}\}_{j=0}^N$, as defined in \cref{sec:embeddings}.

\paragraph{Recommendation policy and user behavior.}
As defined in \cref{section:setting},
$\pi_0$ is set to recommend as $\softmax_x (\hat{r}_{u})$,
and user behavior as simulated in accordance to the discrete dynamics
in \cref{sec:engagement_dynamics}.
This enables us to evaluate and compare counterfactual outcomes under different policies.
Note this entails variation in the $\beta_{ui}$,
meaning there is no single $\beta_u$ that underlies the dynamics:
even in the limit ($\Delta t \rightarrow 0, T \rightarrow \infty$),
user behavior cannot be described by a continuous LV system,
which implies $\varepsilon_{\mathrm{LV}}>0$
(see Fig.~\ref{fig:discrete_and_continuous_lv_simulation}).
Since the baseline RMSE is high,
we set $\beta_{ui} \propto \tilde{r}_{u,x_i}^2$,
where $\tilde{r}_{u,x_i} = \kappa r_{u,x_i} + (1-\kappa) u^\top x$,
so that $\kappa$ interpolates between predicted ratings ($\kappa=0$) and true ratings ($\kappa=1$).
This allows us to (indirectly) control $\varepsilon_{\mathrm{pred}}$.
For simplicity we set $\alpha_u,\gamma_u,\delta_u$ to be fixed.
For all experiments we use $T=\LvExperimentSimulationLength$, and so expect a roughly fixed $\varepsilon_{\mathrm{dev}}>0$.

\paragraph{Methods.}
We compare our approach (\method{LV}) to several baselines:
(i) a \method{default} policy which myopically optimizes for immediate engagement (and so does not break);
(ii) a `safety switch' policy
(\method{safety@$\tau$}) that breaks
once consumption surpasses a threshold $\tau$;
(iii) a prediction-based policy (\method{best-of}) that
chooses the best observed $p_u=\argmax_{p_j} f_j(u)$
(rather than optimizing over $p_u \in [0,1]$);
and (iv) an \method{oracle} benchmark which directly optimizes
the (generally unknown) true underlying dynamics.
We also consider
(v) an adaptive variant of our approach,
\method{LV-adaptive@$T_0$}
(see Sec.~\ref{subsec:beyond_stationary}),
which updates all $p_u$ once at time $T_0<T$ on the basis of
additional user-provided item ratings collected until time $T_0$
(further details in Appendix~\ref{subsec:additional_structure_ratings}).
For all methods
we measure mean long-term engagement rate (LTE),
and report averages and standard errors computed over \LvExperimentNumRepetitions{} random splits.
Performance is measured
relative to the \method{default} baseline
as it represents no change in policy 
(typical absolute values are
$\text{LTE} \approx 10$). 

\subsection{Results and Analysis}
\label{subsection:results_and_analysis}

\paragraph{Main results.}
Figure \ref{fig:experimental_results} (left) compares the performance of our method to other policies.
Here we set $p_0=0$, use 
$N=\LvExperimentMainExperimentNonzeroTreatmentsCount$ with $p_j\in\Set{\LvExperimentMainExperimentNonzeroTreatments}$, and consider $\kappa=\LvExperimentBarplotSelectedKappa$ (note $\kappa$ affects all policies).
As can be seen, our \method{LV} approach significantly improves over \method{default}
(+\LvExperimentOverallImprovementOverMyopic{}\%),
with further mild improvement for
\method{LV-adaptive@$T_0$} with an early $T_0=\LvExperimentAdaptiveTZeroMedium$
(+\LvExperimentAdaptiveImprovementOverMyopic{}\%).
For \method{safety@$\tau$}, 
the advantage of \method{LV}
over the optimal $\tau=\LvExperimentSafetyPolicyThreshold$ (+\LvExperimentOverallImprovementOverSafety{}\%; chosen in hindsight) shows the importance of being preemptive; for the slightly smaller
$\tau=\LvExperimentSafetyPolicyThresholdLow$, breaks are harmful.
The gap from \method{best-of} (+\LvExperimentOverallImprovementOverArgmax{}\%) quantifies the gain from the optimization step in \eqref{eq:optimal_p},
and the close performance to \method{oracle} (\LvExperimentOverallImprovementOverOracle{}\%) suggests that optimizing the empirical curve (\eqref{eq:theta_empirical_fit}) works well as as proxy.
\looseness=-1

\begin{figure}
    \centering
    \includegraphics[width=\linewidth]{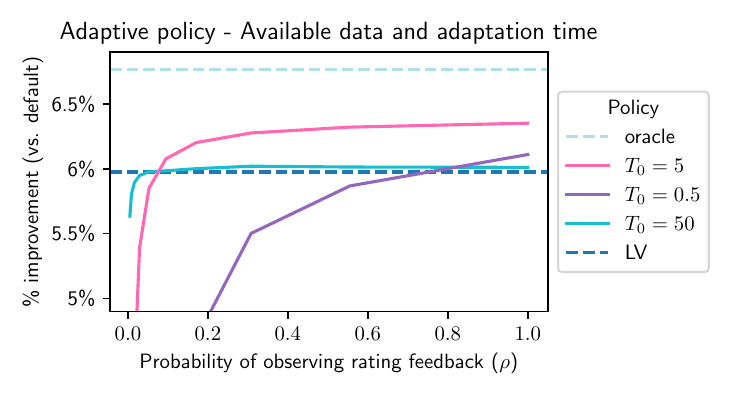}
    \figvspace{}
    \caption{
    Adaptive policy refinement for varying rating density $\rho\in[0,1]$. Solid purple lines represents the adaptive policy described in \cref{subsec:beyond_stationary}, 
    with varying adaptation time $T_0$. 
    Horizontal lines represent the performance of selected non-adaptive policies, as presented in \cref{fig:experimental_results} (Left).
    }
    \label{fig:adaptive_policy}
\end{figure}

\paragraph{User types.}
Figure \ref{fig:experimental_results} (center)
shows for our approach how gain in LTE varies across learned breaking policies $\hat{p}_u > 0$.
For increasingly-accurate predictions ($\kappa \in \{0,0.5,1\}$),
the main plot shows performance gains
for each group of users,
partitioned by their $\hat{p}_u$ values
(binned; plot shows average and unit standard deviation per bin).
Gains until $\hat{p}_u \le 0.15$ are mild,
but for $\hat{p}_u > 0.15$, the general trend is positive:
users who are deemed to require more frequent breaks, benefit more from breaking.
Gains until $\hat{p}_u \le 0.3$ increase for all $\kappa$,
but for $\hat{p}_u>0.3$, extrapolation becomes difficult:
note the higher variation within each $\kappa$,
as well as significant differences across $\kappa$.
This highlights the importance of accurate predictions for inferring optimal $\hat{p}_u$ when the experimental $p_i$ are small.
The inlaied plot shows that,
in line with our theory,
$\hat{p}_u$ exhibits an empirical phase shift in the estimated $\hat{\theta}_u$.\looseness=-1

\newpage
\paragraph{Treatments.}
Figure \ref{fig:experimental_results} (right)
shows the effects of experimental treatments on performance.
Focusing on $N=2$, we fix $p_0=0$,
and consider increasingly aggressive experimentation
by varying
$p_1 \in (0,0.4]$.
For our approach, 
increasing $p_1$ helps,
which is anticipated by our theoretical bound. 
For the \method{best-of} approach, larger $p_1$ also helps,
but exhibits population-level optimum ($p_1 \approx 0.24$), 
which is easy to `overshoot'.
Note that when prediction accuracy is low ($\kappa=0$),
experimentation is essential:
if $p_1$ is not sufficiently large,
then performance can sharply deteriorate.

\paragraph{Adaptive policies.}
\cref{fig:adaptive_policy} compares the performance of \method{LV-adaptive} under different update times $T_0$,
and as a function of the frequency of observing additional user ratings as feedback, $\rho \in [0,1]$.
For a fairly small $T_0=5$ (vs. $T=100$),
even mild $\rho=0.15$ suffices for improving upon the non-adaptive \method{LV}.
Gains grow with increasing $\rho$;
but for $\rho<0.15$, the updated policy relies on too little data to be effective.
This is even more pronounced for an exceedingly premature $T_0=0.5$,
where updates are useful only when $\rho>0.85$.
Interestingly, setting $T_0=50$ also does not help,
since there is insufficient time left to benefit from the update.\squeeze

\section{Discussion}
\label{sec:discussion}

Our paper studies the novel problem of learning optimal breaking policies for recommendation.
We posit a tight connection between long-term engagement and user well-being,
and argue that optimizing the former requires careful consideration of the latter.
Viewing user interest as a limited
we study the role of breaks in facilitating sustainable habits,
and propose an efficient algorithm for learning breaking policies
that optimize long-term engagement.
Our approach relies on LV models at its core,
but incorporating more elaborate dynamic models
is appealing as future work.

The recommendation setting we study is simple,
but offers what we believe is a plausible perspective on the dynamics of user behavior---with emphasis on the importance of bi-directional feedback in shaping outcomes for the system, and for its users.
Nonetheless, further work is necessary to establish the degree to which our stylized model is valid in reality.
Our hopes are that our work takes one step towards establishing
recommendation system as \emph{ecosystems}---requiring
active, planned, and regulated conservation.

\paragraph{Acknowledgements.}
We thank the anonymous reviewers for their helpful remarks and insightful suggestions.
This work is supported by an ISF scholarship (grant No. 278/22).
Eden Saig is supported by a CHE scholarship for Ph.D. students in data science.

\bibliography{references.bib}

\begin{thebibliography}{57}
\providecommand{\natexlab}[1]{#1}
\providecommand{\url}[1]{\texttt{#1}}
\expandafter\ifx\csname urlstyle\endcsname\relax
  \providecommand{\doi}[1]{doi: #1}\else
  \providecommand{\doi}{doi: \begingroup \urlstyle{rm}\Url}\fi

\bibitem[Baumeister et~al.(1998)Baumeister, Bratslavsky, Muraven, and
  Tice]{baumeister1998ego}
Baumeister, R.~F., Bratslavsky, E., Muraven, M., and Tice, D.~M.
\newblock Ego depletion: is the active self a limited resource.
\newblock \emph{Journal of personality and social psychology.}, 74\penalty0
  (5), 1998.

\bibitem[Bergum \& Lehr(1962)Bergum and Lehr]{bergum1962vigilance}
Bergum, B.~O. and Lehr, D.~J.
\newblock Vigilance performance as a function of interpolated rest.
\newblock \emph{Journal of Applied Psychology}, 46\penalty0 (6):\penalty0 425,
  1962.

\bibitem[Cao et~al.(2020)Cao, Sun, Shen, and Ettl]{cao2020fatigue}
Cao, J., Sun, W., Shen, Z.-J.~M., and Ettl, M.
\newblock Fatigue-aware bandits for dependent click models.
\newblock In \emph{AAAI}, pp.\  3341--3348, 2020.

\bibitem[Chaney(2021)]{chaney2021recommendation}
Chaney, A.~J.
\newblock Recommendation system simulations: A discussion of two key
  challenges.
\newblock \emph{arXiv preprint arXiv:2109.02475}, 2021.

\bibitem[Chaney et~al.(2018)Chaney, Stewart, and
  Engelhardt]{chaney2018algorithmic}
Chaney, A.~J., Stewart, B.~M., and Engelhardt, B.~E.
\newblock How algorithmic confounding in recommendation systems increases
  homogeneity and decreases utility.
\newblock In \emph{Proceedings of the 12th ACM conference on recommender
  systems}, pp.\  224--232, 2018.

\bibitem[Chen \& Plemmons(2010)Chen and Plemmons]{chen2010nonnegativity}
Chen, D. and Plemmons, R.~J.
\newblock Nonnegativity constraints in numerical analysis.
\newblock In \emph{The birth of numerical analysis}, pp.\  109--139. World
  Scientific, 2010.

\bibitem[Constine(2018)]{constine2018instagram}
Constine, J.
\newblock Instagram says ‘you’re all caught up’ in first time-well-spent
  feature. techcrunch, 2018.
\newblock URL \url{https://techcrunch.com/2018/05/21/scroll-responsibly/}.
\newblock [Online; accessed 23-September-2022].

\bibitem[Curmei et~al.(2022)Curmei, Haupt, Recht, and
  Hadfield-Menell]{curemi2022towards}
Curmei, M., Haupt, A.~A., Recht, B., and Hadfield-Menell, D.
\newblock Towards psychologically-grounded dynamic preference models.
\newblock In \emph{Proceedings of the 16th ACM Conference on Recommender
  Systems}, RecSys '22, pp.\  35–48, New York, NY, USA, 2022. Association for
  Computing Machinery.
\newblock ISBN 9781450392785.
\newblock \doi{10.1145/3523227.3546778}.

\bibitem[Danziger et~al.(2011)Danziger, Levav, and
  Avnaim-Pesso]{danziger2011extraneous}
Danziger, S., Levav, J., and Avnaim-Pesso, L.
\newblock Extraneous factors in judicial decisions.
\newblock \emph{Proceedings of the National Academy of Sciences}, 108\penalty0
  (17):\penalty0 6889--6892, 2011.
\newblock \doi{10.1073/pnas.1018033108}.

\bibitem[Dean \& Morgenstern(2022)Dean and Morgenstern]{dean2022preference}
Dean, S. and Morgenstern, J.
\newblock Preference dynamics under personalized recommendations.
\newblock In \emph{Proceedings of the 23rd ACM Conference on Economics and
  Computation}, EC '22, pp.\  795–816, New York, NY, USA, 2022. Association
  for Computing Machinery.
\newblock ISBN 9781450391504.
\newblock \doi{10.1145/3490486.3538346}.

\bibitem[Dean et~al.(2020)Dean, Curmei, and Recht]{dean2020designing}
Dean, S., Curmei, M., and Recht, B.
\newblock Designing recommender systems with reachability in mind.
\newblock In \emph{Workshop on Participatory Approaches to Machine Learning},
  2020.

\bibitem[Ding et~al.(2016)Ding, Xu, Chen, and Xu]{ding2016beyond}
Ding, X., Xu, J., Chen, G., and Xu, C.
\newblock Beyond smartphone overuse: identifying addictive mobile apps.
\newblock In \emph{Proceedings of the 2016 CHI Conference Extended Abstracts on
  Human Factors in Computing Systems}, pp.\  2821--2828, 2016.

\bibitem[Du et~al.(2015)Du, Wang, He, Sun, and Song]{du2015time}
Du, N., Wang, Y., He, N., Sun, J., and Song, L.
\newblock Time-sensitive recommendation from recurrent user activities.
\newblock \emph{Advances in neural information processing systems}, 28, 2015.

\bibitem[Dubey et~al.(2022)Dubey, Griffiths, and Dayan]{dubey2022pursuit}
Dubey, R., Griffiths, T.~L., and Dayan, P.
\newblock The pursuit of happiness: {A} reinforcement learning perspective on
  habituation and comparisons.
\newblock \emph{PLoS Comput. Biol.}, 18\penalty0 (8), 2022.
\newblock \doi{10.1371/journal.pcbi.1010316}.

\bibitem[Duncan et~al.(2019)Duncan, Aubele-Futch, and McGrath]{duncan2019fast}
Duncan, J.~P., Aubele-Futch, T., and McGrath, M.
\newblock A fast-slow dynamical system model of addiction: Predicting relapse
  frequency.
\newblock \emph{SIAM Journal on Applied Dynamical Systems}, 18\penalty0
  (2):\penalty0 881--903, 2019.

\bibitem[Elhai et~al.(2017)Elhai, Dvorak, Levine, and
  Hall]{elhai2017problematic}
Elhai, J.~D., Dvorak, R.~D., Levine, J.~C., and Hall, B.~J.
\newblock Problematic smartphone use: A conceptual overview and systematic
  review of relations with anxiety and depression psychopathology.
\newblock \emph{Journal of affective disorders}, 207:\penalty0 251--259, 2017.

\bibitem[Eyal(2019)]{eyal2019hooked}
Eyal, N.
\newblock \emph{Hooked : how to build habit-forming products / Nir Eyal with
  Ryan Hoover.}
\newblock Portfolio/Penguin, New York, updated edition edition, 2019.
\newblock ISBN 9781591847786.

\bibitem[Farmer(2002)]{farmer2002market}
Farmer, J.~D.
\newblock Market force, ecology and evolution.
\newblock \emph{Industrial and Corporate Change}, 11\penalty0 (5):\penalty0
  895--953, 2002.

\bibitem[Gilboa et~al.(2008)Gilboa, Shirom, Fried, and Cooper]{gilboa2008meta}
Gilboa, S., Shirom, A., Fried, Y., and Cooper, C.
\newblock A meta-analysis of work demand stressors and job performance:
  examining main and moderating effects.
\newblock \emph{Personnel psychology}, 61\penalty0 (2):\penalty0 227--271,
  2008.

\bibitem[Gorbach et~al.(2017)Gorbach, Bauer, and Buhmann]{gorbach2017scalable}
Gorbach, N.~S., Bauer, S., and Buhmann, J.~M.
\newblock Scalable variational inference for dynamical systems.
\newblock \emph{Advances in neural information processing systems}, 30, 2017.

\bibitem[Guo et~al.(2021)Guo, Krauth, Jordan, and Garg]{guo2021stereotyping}
Guo, W., Krauth, K., Jordan, M., and Garg, N.
\newblock The stereotyping problem in collaboratively filtered recommender
  systems.
\newblock In \emph{Equity and Access in Algorithms, Mechanisms, and
  Optimization}, pp.\  1--10. ACM, 2021.

\bibitem[Harper \& Konstan(2015)Harper and Konstan]{harper2015movielens}
Harper, F.~M. and Konstan, J.~A.
\newblock The movielens datasets: History and context.
\newblock \emph{Acm transactions on interactive intelligent systems (tiis)},
  5\penalty0 (4):\penalty0 1--19, 2015.

\bibitem[Helton \& Russell(2017)Helton and Russell]{helton2017rest}
Helton, W.~S. and Russell, P.~N.
\newblock Rest is still best: The role of the qualitative and quantitative load
  of interruptions on vigilance.
\newblock \emph{Human factors}, 59\penalty0 (1):\penalty0 91--100, 2017.

\bibitem[Hennfng et~al.(1989)Hennfng, Sauter, Salvendy, and
  Krieg~Jr]{hennfng1989microbreak}
Hennfng, R.~A., Sauter, S.~L., Salvendy, G., and Krieg~Jr, E.~F.
\newblock Microbreak length, performance, and stress in a data entry task.
\newblock \emph{Ergonomics}, 32\penalty0 (7):\penalty0 855--864, 1989.

\bibitem[Hofbauer et~al.(1998)Hofbauer, Sigmund,
  et~al.]{hofbauer1998evolutionary}
Hofbauer, J., Sigmund, K., et~al.
\newblock \emph{Evolutionary games and population dynamics}.
\newblock Cambridge university press, 1998.

\bibitem[Hron et~al.(2022)Hron, Krauth, Jordan, Kilbertus, and
  Dean]{hron2022modeling}
Hron, J., Krauth, K., Jordan, M.~I., Kilbertus, N., and Dean, S.
\newblock Modeling content creator incentives on algorithm-curated platforms.
\newblock \emph{arXiv preprint arXiv:2206.13102}, 2022.

\bibitem[Hug(2020)]{Hug2020}
Hug, N.
\newblock Surprise: A python library for recommender systems.
\newblock \emph{Journal of Open Source Software}, 5\penalty0 (52):\penalty0
  2174, 2020.

\bibitem[Ie et~al.(2019)Ie, Hsu, Mladenov, Jain, Narvekar, Wang, Wu, and
  Boutilier]{ie2019recsim}
Ie, E., Hsu, C.-w., Mladenov, M., Jain, V., Narvekar, S., Wang, J., Wu, R., and
  Boutilier, C.
\newblock Recsim: A configurable simulation platform for recommender systems.
\newblock \emph{arXiv preprint arXiv:1909.04847}, 2019.

\bibitem[Kahneman(1973)]{kahneman1973attention}
Kahneman, D.
\newblock \emph{Attention and effort}, volume 1063.
\newblock Citeseer, 1973.

\bibitem[Kalimeris et~al.(2021)Kalimeris, Bhagat, Kalyanaraman, and
  Weinsberg]{kalimeris2021preference}
Kalimeris, D., Bhagat, S., Kalyanaraman, S., and Weinsberg, U.
\newblock Preference amplification in recommender systems.
\newblock In \emph{Proceedings of the 27th ACM SIGKDD Conference on Knowledge
  Discovery \& Data Mining}, pp.\  805--815, 2021.

\bibitem[Kleinberg et~al.(2022)Kleinberg, Mullainathan, and
  Raghavan]{kleinberg2022challenge}
Kleinberg, J., Mullainathan, S., and Raghavan, M.
\newblock The challenge of understanding what users want: Inconsistent
  preferences and engagement optimization.
\newblock In \emph{Proceedings of the 23rd ACM Conference on Economics and
  Computation}, pp.\  29--29, 2022.

\bibitem[Kleinberg \& Immorlica(2018)Kleinberg and
  Immorlica]{kleinberg2018recharging}
Kleinberg, R. and Immorlica, N.
\newblock Recharging bandits.
\newblock In \emph{2018 IEEE 59th Annual Symposium on Foundations of Computer
  Science (FOCS)}, pp.\  309--319. IEEE, 2018.

\bibitem[Krauth et~al.(2020)Krauth, Dean, Zhao, Guo, Curmei, Recht, and
  Jordan]{krauth2020offline}
Krauth, K., Dean, S., Zhao, A., Guo, W., Curmei, M., Recht, B., and Jordan,
  M.~I.
\newblock Do offline metrics predict online performance in recommender systems?
\newblock \emph{arXiv preprint arXiv:2011.07931}, 2020.

\bibitem[Lee et~al.(2014)Lee, Chang, Lin, and Cheng]{lee2014dark}
Lee, Y.-K., Chang, C.-T., Lin, Y., and Cheng, Z.-H.
\newblock The dark side of smartphone usage: Psychological traits, compulsive
  behavior and technostress.
\newblock \emph{Computers in human behavior}, 31:\penalty0 373--383, 2014.

\bibitem[Leqi et~al.(2021)Leqi, Kilinc~Karzan, Lipton, and
  Montgomery]{leqi2021rebounding}
Leqi, L., Kilinc~Karzan, F., Lipton, Z., and Montgomery, A.
\newblock Rebounding bandits for modeling satiation effects.
\newblock \emph{Advances in Neural Information Processing Systems},
  34:\penalty0 4003--4014, 2021.

\bibitem[Lotka(1910)]{lotka1910contribution}
Lotka, A.~J.
\newblock Contribution to the theory of periodic reactions.
\newblock \emph{The Journal of Physical Chemistry}, 14\penalty0 (3):\penalty0
  271--274, 1910.
\newblock \doi{10.1021/j150111a004}.

\bibitem[Mansoury et~al.(2020)Mansoury, Abdollahpouri, Pechenizkiy, Mobasher,
  and Burke]{mansoury2020feedback}
Mansoury, M., Abdollahpouri, H., Pechenizkiy, M., Mobasher, B., and Burke, R.
\newblock Feedback loop and bias amplification in recommender systems.
\newblock In \emph{Proceedings of the 29th ACM international conference on
  information \& knowledge management}, pp.\  2145--2148, 2020.

\bibitem[Montag et~al.(2018)Montag, Zhao, Sindermann, Xu, Fu, Li, Zheng, Li,
  Kendrick, Dai, et~al.]{montag2018internet}
Montag, C., Zhao, Z., Sindermann, C., Xu, L., Fu, M., Li, J., Zheng, X., Li,
  K., Kendrick, K.~M., Dai, J., et~al.
\newblock Internet communication disorder and the structure of the human brain:
  Initial insights on wechat addiction.
\newblock \emph{Scientific reports}, 8\penalty0 (1):\penalty0 1--10, 2018.

\bibitem[Mullette-Gillman et~al.(2015)Mullette-Gillman, Leong, and
  Kurnianingsih]{mullette2015cognitive}
Mullette-Gillman, O.~A., Leong, R.~L., and Kurnianingsih, Y.~A.
\newblock Cognitive fatigue destabilizes economic decision making preferences
  and strategies.
\newblock \emph{PloS one}, 10\penalty0 (7):\penalty0 e0132022, 2015.

\bibitem[Muraven \& Baumeister(2000)Muraven and Baumeister]{muraven2000self}
Muraven, M. and Baumeister, R.~F.
\newblock Self-regulation and depletion of limited resources: Does self-control
  resemble a muscle?
\newblock \emph{Psychological bulletin}, 126\penalty0 (2):\penalty0 247, 2000.

\bibitem[Perez(2018)]{perez2018apple}
Perez, S.
\newblock Apple unveils a new set of ‘digital wellness’ features for better
  managing screen time.
\newblock \emph{TechCrunch, June}, 4, 2018.
\newblock URL
  \url{https://techcrunch.com/2018/06/04/apple-unveils-a-new-set-of-digital-wellness-features-for-better-managing-screen-time/}.
\newblock [Online; accessed 23-September-2022].

\bibitem[Rauch(2018)]{rauch2018slow}
Rauch, J.
\newblock \emph{Slow media: Why slow is satisfying, sustainable, and smart}.
\newblock Oxford University Press, 2018.

\bibitem[Robertson(1977)]{robertson1977probability}
Robertson, S.~E.
\newblock The probability ranking principle in ir.
\newblock \emph{Journal of documentation}, 1977.

\bibitem[Ross et~al.(2014)Ross, Russell, and Helton]{ross2014effects}
Ross, H.~A., Russell, P.~N., and Helton, W.~S.
\newblock Effects of breaks and goal switches on the vigilance decrement.
\newblock \emph{Experimental brain research}, 232\penalty0 (6):\penalty0
  1729--1737, 2014.

\bibitem[Ryder et~al.(2018)Ryder, Golightly, McGough, and
  Prangle]{ryder2018black}
Ryder, T., Golightly, A., McGough, A.~S., and Prangle, D.
\newblock Black-box variational inference for stochastic differential
  equations.
\newblock In \emph{International Conference on Machine Learning}, pp.\
  4423--4432. PMLR, 2018.

\bibitem[Samuelson(1998)]{samuelson1998evolutionary}
Samuelson, L.
\newblock \emph{Evolutionary games and equilibrium selection}, volume~1.
\newblock MIT press, 1998.

\bibitem[Sanna~Passino et~al.(2021)Sanna~Passino, Maystre, Moor, Anderson, and
  Lalmas]{passino2021where}
Sanna~Passino, F., Maystre, L., Moor, D., Anderson, A., and Lalmas, M.
\newblock Where to next? a dynamic model of user preferences.
\newblock In \emph{Proceedings of the Web Conference 2021}, WWW '21, pp.\
  3210–3220, New York, NY, USA, 2021. Association for Computing Machinery.
\newblock ISBN 9781450383127.
\newblock \doi{10.1145/3442381.3450028}.

\bibitem[Schmit \& Riquelme(2018)Schmit and Riquelme]{schmit2018human}
Schmit, S. and Riquelme, C.
\newblock Human interaction with recommendation systems.
\newblock In \emph{International Conference on Artificial Intelligence and
  Statistics}, pp.\  862--870. PMLR, 2018.

\bibitem[Scholl et~al.(2021)Scholl, Calinescu, and
  Farmer]{doi:10.1073/pnas.2015574118}
Scholl, M.~P., Calinescu, A., and Farmer, J.~D.
\newblock How market ecology explains market malfunction.
\newblock \emph{Proceedings of the National Academy of Sciences}, 118\penalty0
  (26):\penalty0 e2015574118, 2021.
\newblock \doi{10.1073/pnas.2015574118}.

\bibitem[Sievertsen et~al.(2016)Sievertsen, Gino, and
  Piovesan]{sievertsen2016cognitive}
Sievertsen, H.~H., Gino, F., and Piovesan, M.
\newblock Cognitive fatigue influences students’ performance on standardized
  tests.
\newblock \emph{Proceedings of the National Academy of Sciences}, 113\penalty0
  (10):\penalty0 2621--2624, 2016.

\bibitem[Takeuchi(1996)]{takeuchi1996global}
Takeuchi, Y.
\newblock \emph{Global dynamical properties of Lotka-Volterra systems}.
\newblock World Scientific, 1996.

\bibitem[Tr{\'e}lat \& Zuazua(2015)Tr{\'e}lat and Zuazua]{trelat2015turnpike}
Tr{\'e}lat, E. and Zuazua, E.
\newblock The turnpike property in finite-dimensional nonlinear optimal
  control.
\newblock \emph{Journal of Differential Equations}, 258\penalty0 (1):\penalty0
  81--114, 2015.

\bibitem[Wan \& McAuley(2018)Wan and McAuley]{wan2018item}
Wan, M. and McAuley, J.
\newblock Item recommendation on monotonic behavior chains.
\newblock In \emph{Proceedings of the 12th ACM conference on recommender
  systems}, pp.\  86--94, 2018.

\bibitem[Wan et~al.(2019)Wan, Misra, Nakashole, and McAuley]{wan2019fine}
Wan, M., Misra, R., Nakashole, N., and McAuley, J.
\newblock Fine-grained spoiler detection from large-scale review corpora.
\newblock In \emph{Proceedings of the 57th Annual Meeting of the Association
  for Computational Linguistics}, pp.\  2605--2610, 2019.

\bibitem[Wang \& Lin(2003)Wang and Lin]{wang2003personalization}
Wang, J.-C. and Lin, J.-P.
\newblock Are personalization systems really personal?-effects of conformity in
  reducing information.
\newblock In \emph{36th Annual Hawaii International Conference on System
  Sciences, 2003. Proceedings of the}, pp.\  10--pp. IEEE, 2003.

\bibitem[Warlop et~al.(2018)Warlop, Lazaric, and Mary]{warlop2018fighting}
Warlop, R., Lazaric, A., and Mary, J.
\newblock Fighting boredom in recommender systems with linear reinforcement
  learning.
\newblock \emph{Advances in Neural Information Processing Systems}, 31, 2018.

\bibitem[Weibull(1997)]{weibull1997evolutionary}
Weibull, J.~W.
\newblock \emph{Evolutionary game theory}.
\newblock MIT press, 1997.

\end{thebibliography}
\bibliographystyle{icml2023}

\newpage
\appendix
\onecolumn

\section{Ethics and Limitations}
Our works concerns the task of optimizing long-term user engagement.
For this, we give and efficient learning framework,
supported by conductive theoretical results and favorable empirical evaluation.
But since in reality any breaking policy involves, and can therefore affect, real human users, 
it is important to understand its limitations.
The principle underlying our work is the idea of `recommendation as conservation':
if systems wish to remain relevant to their users in the long term must,
then they must also actively promote their users' well-being.
We believe that breaks---with appropriate execution---have potential to be one such means.
Still, this should not be taken to mean that optimizing breaking policies will always and necessarily improve well-being.
In theory, we can see how breaking schedules can also be used nefariously for dictating consumption habits;
e.g., by enabling temporally-varying rewards that 
harness psychological weaknesses
to \emph{foster} addictive habits
(e.g., as part of a mechanism for `hooking' users \citep{eyal2019hooked}).
As such, breaks should be administered transparently
and with proper evaluation of how they actually contribute to user well-being.

To study the role of bi-directional feedback in recommender systems, 
and as is common in the growing literature on recommendation dynamics,
our work abstracts away certain notions of the recommendation process to enable focus on others.
In particular, our work considers a simple yet plausible model of user behavior,
and to this model our theoretical and empirical results apply;
note that the model does not account for social interaction, exogenous temporal effects (e.g., weekday, seasonal), variability or innovation in content, or other forms of user feedback, either implicit or explicit.
Nonetheless, a main benefit of our model lies in the connection we make to LV dynamics---this promoting discussion regarding consumption under limited resources and the need for sustainability in recommendation.
We leave the study of more elaborate LV systems
as better-informed models of recommendation dynamics for future work.

\section{Deferred Proofs}
In this section, we formalize the model presented in \cref{section:setting}, and prove the claims presented in \cref{section:continuous_lv}. The section ends with a formal proof of \cref{thm:bound}.

\subsection{Asymptotic Empirical Rate and Harmonic Mean}
\label{section:empirical_rate_harmonic_mean}
This section formalizes the transition from optimization of time-gaps $\Delta t_i$ to optimization of instantaneous rates $\lambda_i=\Delta t_i^{-1}$. We denote the time horizon by $T$, the empirical engagement rate by $\tfrac{1}{T}\Size{S_u} = \tfrac{1}{T}\Size{\Set{(t_i, x_i, r_i)\mid t_i \le T}}$ (see \eqref{eq:interaction_sequence}), the time gaps by $\Delta t_i = t_{i+1}-t_i$, and the instanteneous rate by $\lambda_i=\Delta_i^{-1}$ (see \eqref{eq:instantaneous_rate}).

\begin{proposition}    
Maximizing the empirical rate 
$\tfrac{1}{T}\Size{S_u}$
is asymptotically
equivalent to maximizing the \emph{harmonic mean} of $\lambda_i$. 
\end{proposition}
\begin{proof}
The last interaction event within the time horizon is $k=\argmax_i \{t_i\mid t_i\le T\}$, and the time gap between the last event and the time horizon is $\varepsilon = T-t_k$.
Using this notation, the rate of events is equal to $\tfrac{1}{T}|S_u|=\tfrac{k}{t_k + \varepsilon}$.
In the long-term limit $T\to\infty$, which is our focus, $\varepsilon$ is negligible compared to $t_k$ and therefore the rate of events converges towards $\tfrac{k}{t_k}$.
Since $t_k = \sum_{i=1}^k \Delta t_i$ as a telescopic sum, we have $\frac{1}{T}|S_u| \approx \left(\tfrac{\sum_{i=1}^k \Delta t_i}{k}\right)^{-1}$.
As the relation is inverse, the rate of events ${1}{T}|S_u|$ is maximized by minimizing the average time between events $\Delta t_i = t_{i+1}-t_i$, or equivalently maximizing the harmonic mean of $\lambda_i = \Delta t_i^{-1}$.
\end{proof}

\subsection{Properties of Lotka-Volterra Systems}
\label{section:single_channel_lv_properties}

\begin{definition}[Static policy equilibrium]
Let $\lambda(t; \theta), z(t; \theta)$ denote a Lokta-Volterra model characterized by parameters $\theta=(\alpha, \beta, \gamma, \delta)\in\Reals^4_+$, as defined in \Eqref{eq:lv}. Let $p\in[0,1]$, and denote by $\pi_p$ the static policy corresponding to $p$. For $\lambda(0), z(0)>0$, the static equilibrium of the system is defined as:
\begin{align*}
    \lambda^*(p;\theta)
    =
    \lim_{t\to\infty} \lambda(t; \theta)
    ; \quad\quad
    z^*(p;\theta)
    =
    \lim_{t\to\infty} z(t; \theta)
\end{align*}
\end{definition}
We denote $\lambda^*(p)=\lambda^*(p;\theta)$ when $\theta$ is clear from the context. We denote $\lambda^*(p;u)=\lambda^*(p;\theta_u)$ when a user $u\in\Users$ characterized by parameters $\theta_u$ is given and clear from the context.

\begin{proposition}[Global stability]
\label{claim:lv_equilibrium_exists}
$\lambda^*(p;\theta)$ exists and uniquely defined for all $\theta\in\Reals_+^4$, $p\in[0,1]$ and for all initial conditions $\lambda(0), z(0)>0$.
\end{proposition}
\begin{proof}
See \citep[Section~3.2]{takeuchi1996global}.
\end{proof}

\begin{lemma}[Equilibrium of LV behavioral model. Formal proof of \cref{lemma:control_equilibrium}]
\label{claim:lv_one_channel_equilibrium}
Assume a Lokta-Volterra model characterized by $\theta=(\alpha, \beta, \gamma, \delta)\in\Reals_+^4$, and let $p\in[0,1]$ denote the proportion of interactions in which a forced break is served. The static equilibrium of the model under static policy $\pi_p$ is given by:
\begin{align*}
    \lambda^*(p)&=\begin{cases}
    \frac{\gamma}{\delta}\frac{1}{1-p}\left(1-\frac{\alpha}{\beta}\frac{1}{1-p}\right)
    & p\in\left[0,1-\frac{\alpha}{\beta}\right]
    \\
    0 & \mathrm{otherwise}
    \end{cases}
    \\
    z^*(p)&=\begin{cases}
    \frac{\alpha}{\beta}\frac{1}{1-p}
    & p\in\left[0,1-\frac{\alpha}{\beta}\right]
    \\
    1 & \mathrm{otherwise}
    \end{cases}
\end{align*}
\end{lemma}
\begin{proof}
The LV dynamical system is given by \Eqref{eq:lv}:
\begin{align*}
    \FullDerivative{\lambda}{t} &= -\alpha\lambda + \beta(1-p)\lambda z \\
    \FullDerivative{z}{t} &= \gamma z(1-z) - \delta(1-p)\lambda z
\end{align*}
when $p\in\left[0,1-\frac{\alpha}{\beta}\right]$ we equate $\FullDerivative{\lambda}{t}=0$, $\FullDerivative{z}{t}=0$ and obtain the result. The solution is guaranteed to be valid, as both $\lambda^*(p)\ge0$ and $z^*(p)\in[0,1]$.

Conversely, when $p\notin\left[0,1-\frac{\alpha}{\beta}\right]$, there exists $\epsilon>0$ such that $\FullDerivative{}{t}\log \lambda<-\epsilon<0$ for all $\lambda>0$, $z\in[0,1]$. 
Thus, $\log \lambda(t)$ tends towards $-\infty$ over time, and therefore $\lambda(t)$ tends towards $0$, and $\lambda^*(p)=0$ as required. When $\lambda(t)$ is close to zero, the interaction terms vanish in the $\FullDerivative{z}{t}$ equation, and $z(t)$ grows logistically towards $1$.
\end{proof}

\begin{proposition}[Equilibrium bounds]
\label{claim:lv_one_channel_equilibrium_bounds}
For a Lotka-Volterra model, the static policy equilibrium $\lambda^*(p)$ is bounded by:
$$
0 \le \lambda^*(p) \le\frac{\beta\gamma}{4\alpha\delta}
$$
\end{proposition}
\begin{proof}
Denote $x=\frac{1}{1-p}$. From \cref{claim:lv_one_channel_equilibrium}, for $x\in\left[1,\frac{\beta}{\alpha}\right]$ the equilibrium consumption $\lambda^*(x)$ is given by:
$$
\lambda^*(x)
=\frac{\gamma}{\delta}x\left(1-\frac{\alpha}{\beta}x\right)
$$
and is zero otherwise. The equilibrium is a quadratic function of $x$ with roots $x\in\Set{0,\frac{\beta}{\alpha}}$, and therefore attains its maximum at $x=\frac{\beta}{2\alpha}$. Plugging back the maximizing $x$ into $\lambda^*$ we obtain the upper bound.
Lower bound is attained as the equilibrium in \cref{claim:lv_one_channel_equilibrium} is clipped by $0$ from below.
\end{proof}

\begin{lemma}[Optimal static policy. Formal proof of \cref{lemma:optimal_p_informal}]
\label{claim:lv_one_channel_optimal_policy}
The optimal static policy for a Lotka-Volterra system is given by:
\begin{equation*}
    p_\mathrm{opt} = \begin{cases}
    1-2\frac{\alpha}{\beta} & \frac{\alpha}{\beta}\le\frac{1}{2} \\
    0 & \frac{\alpha}{\beta}>\frac{1}{2}
    \end{cases}
\end{equation*}
And the optimal equilibrium engagement rate is given by:
\begin{equation*}
    \lambda^*_\mathrm{opt} = \begin{cases}
    \frac{\beta\gamma}{4\alpha\delta} & \frac{\alpha}{\beta}\le\frac{1}{2} \\
    \frac{\gamma}{\delta}\left(1-\frac{\alpha}{\beta}\right) & \frac{\alpha}{\beta}>\frac{1}{2}
    \end{cases}
\end{equation*}
\end{lemma}
\begin{proof}
Denote $x=\frac{1}{1-p}$. From \cref{claim:lv_one_channel_equilibrium_bounds}, the global maximum of $\lambda^*(x)$ is attained at $x=\frac{\beta}{2\alpha}$. 
Consider two cases: 
When $\frac{\alpha}{\beta}\le\frac{1}{2}$, we obtain that $x_\mathrm{opt}=\frac{\beta}{2\alpha}\ge 1$, and therefore $p_\mathrm{opt}=1-\frac{1}{x}\in[0,1]$. From this we obtain that in this case the global maximum is attained on the simplex, and given by the formula from \cref{claim:lv_one_channel_equilibrium_bounds}. 
Conversely, when $\frac{\alpha}{\beta}>\frac{1}{2}$, we obtain $p=1-\frac{1}{x}<0$, and therefore $x_\mathrm{opt}$ translates to a negative value of $p$. As $\lambda^*(p)$ is uni-modal, the optimal policy restricted to the simplex $[0,1]$ in this case is attained on the closest boundary point $p=0$. 

\cref{fig:lv_equilibrium} provides graphical intuition for this proof.
\end{proof}

\begin{proposition}[Inference of $\alpha/\beta$ from two-treatment equilibrium data. Formal proof of \cref{prop:realizable_closed_form}]
\label{lemma:lv_one_channel_alpha_beta_estimator}
Let $\lambda(t),z(t)$ be a Lokta-Volterra model, let $p_1, p_2\in[0,1]$. Denote by $\lambda^*(p_1), \lambda^*(p_2)$ the static equilibrium rates corresponding to static policies $\pi_{p_1}, \pi_{p_2}$, and assume $\lambda^*(p_1), \lambda^*(p_2)>0$. The parameter ratio $\frac{\alpha}{\beta}$ is given by the following formula:
\begin{equation*}
    \frac{\alpha}{\beta}
    =
    \frac{(1-p_2)\lambda^*(p_2)-(1-p_1)\lambda^*(p_1)}{\frac{1}{1-p_1}-\frac{1}{1-p_2}}
\end{equation*}
\end{proposition}
\begin{proof}
From \cref{claim:lv_one_channel_equilibrium}, the equilibrium consumption $\lambda^*(p)$ is given by:
\begin{align*}
    \lambda^*(p)
    &=\frac{\gamma}{\delta}\frac{1}{1-p}\left(1-\frac{\alpha}{\beta}\frac{1}{1-p}\right) \\
    &=\frac{\gamma}{\delta}\frac{1}{1-p}-\frac{\alpha}{\beta}\frac{\gamma}{\delta}\left(\frac{1}{1-p}\right)^2
\end{align*}

When $\lambda^*(p_i)$ is observed for different policies $p_1,\dots,p_m\in\left[0,1-\frac{\alpha}{\beta}\right]$, we obtain a polynomial regression problem for the parameters $\frac{\alpha}{\beta}$ and $\frac{\alpha}{\beta}\frac{\gamma}{\delta}$, which can be solved e.g using Non-Negative Least Squares. 

When $m=2$, we obtain a system of two linear equations. Apply Cramer's rule to obtain:
\begin{align}
    \label{eq:lv_one_channel_gamma_delta_estimator}
    \frac{\gamma}{\delta}
    &=
    \frac
    {\frac{\lambda^*(p_2)}{(1-p_1)^2}-\frac{\lambda^*(p_1)}{(1-p_2)^2}}
    {\frac{1}{(1-p_1)(1-p_2)^2}-\frac{1}{(1-p_1)^2(1-p_2)}} 
    =
    \frac
    {(1-p_2)^2\lambda^*(p_2)-(1-p_1)^2\lambda^*(p_1)}
    {p_2-p_1}
    \\
    \nonumber
    \\
    \frac{\alpha}{\beta}\frac{\gamma}{\delta}
    &=
    \frac
    {\frac{\lambda^*(p_2)}{(1-p_1)}-\frac{\lambda^*(p_1)}{(1-p_2)}}
    {\frac{1}{(1-p_1)(1-p_2)^2}-\frac{1}{(1-p_1)^2(1-p_2)}}
    =(1-p_1)(1-p_2)\frac
    {(1-p_2)\lambda^*(p_2)-(1-p_1)\lambda^*(p_1)}
    {p_2-p_1}
\end{align}
And therefore $\frac{\alpha}{\beta}$ is given by:
\begin{align*}
    \frac{\alpha}{\beta}
    &= \frac
    {\frac{\lambda^*(p_2)}{(1-p_1)}-\frac{\lambda^*(p_1)}{(1-p_2)}}
    {\frac{\lambda^*(p_2)}{(1-p_1)^2}-\frac{\lambda^*(p_1)}{(1-p_2)^2}}
    = (1-p_1)(1-p_2)\frac
    {(1-p_2)\lambda^*(p_2)-(1-p_1)\lambda^*(p_1)}
    {(1-p_2)^2\lambda^*(p_2)-(1-p_1)^2\lambda^*(p_1)}
\end{align*}
\end{proof}

\subsection{Model Fitting From Engagement Predictions}
\paragraph{Notations.} In this section only, we use the common notation $q=1-p$ to denote complementary probabilities.

\label{section:appendix_model_fit}
\begin{definition}[Empirical value of $\alpha/\beta$]
\label{definition:single_channel_alpha_beta_estimator}
For single-channel experiments with forced-break probabilities $p_1,p_2$, denote $\lambda_i=\lambda^*(p_i)$, $f_i=f_{p_i}(u)$, $q_i=1-p_i$. The empirical value of the $\frac{\alpha}{\beta}$ parameter is given by the following formula:
\begin{align*}
    \hat{\frac{\alpha}{\beta}}=\frac{q_1q_2\left(q_1f_1-q_2f_2\right)}{q_1^2f_1-q_2^2f_2}
\end{align*}
\end{definition}

\begin{proposition}[$\alpha/\beta$ estimation error from prediction errors]
\label{claim:single_channel_alpha_beta_estimation_error_bound}
Given a single-channel Lokta-Volterra system with parameter $\frac{\alpha}{\beta}\ge1$. Let $p_1,p_2\in\left[1,\frac{\alpha}{\beta}\right]$, denote $\lambda^*_i=\lambda^*(p_i)\in\Reals_+$, and let $f_i=\lambda^*_i+\varepsilon_i$ be the predicted engagement rates corresponding to $p_1,p_2$. When $\Abs{\varepsilon_1},\Abs{\varepsilon_2}\le\varepsilon\le\frac{\gamma}{\delta}\frac{\Abs{p_1-p_2}}{4}$, the estimation error is bounded by:
\begin{equation*}
    \Abs{\frac{\alpha}{\beta} - \hat{\frac{\alpha}{\beta}}}
    \le
    \frac{\varepsilon}{\Abs{p_1-p_2}}\frac{\beta\delta}{\alpha\gamma}
\end{equation*}
\end{proposition}
\begin{proof}
denote $q_i=1-p_i$. The value of $\frac{\alpha}{\beta}$ is given by \cref{lemma:lv_one_channel_alpha_beta_estimator}:
$$
\frac{\alpha}{\beta} = \frac{q_1q_2\left(q_1\lambda^*_1-q_2\lambda^*_2\right)}{q_1^2\lambda^*_1-q_2^2\lambda^*_2}
$$
And the estimator for $\frac{\alpha}{\beta}$ is obtained by replacing the true value with their predictions:
\begin{align*}
\hat{\frac{\alpha}{\beta}}
&= \frac{q_1q_2\left(q_1f_1-q_2f_2\right)}{q_1^2f_1-q_2^2f_2} \\
&= \frac{q_1q_2\left(q_1(\lambda^*_1+\varepsilon_1)-q_2(\lambda^*_2+\varepsilon_2)\right)}{q_1^2(\lambda^*_1+\varepsilon_1)-q_2^2(\lambda^*_2+\varepsilon_2)}
\end{align*}
The estimation error is given by:
\begin{align*}
\Abs{\frac{\alpha}{\beta} - \hat{\frac{\alpha}{\beta}}}
&=
\Abs{\frac
{q_1^2 q_2^2 (q_1-q_2)(\varepsilon_2\lambda^*_1-\varepsilon_1\lambda^*_2)}
{(q_1^2\lambda^*_1-q_2^2\lambda^*_2)(q_1^2\lambda^*_1-q_2^2\lambda^*_2 - (q_1^2\varepsilon_1-q_2^2\varepsilon_2))}
}
\\
&=
\underbrace{
    \left(q_1q_2\right)^2
}_{\equiv\text{(i)}}
\underbrace{
    \Abs{\frac{q_1-q_2}{q_1^2\lambda^*_1-q_2^2\lambda^*_2}}
}_{\equiv\text{(ii)}}
\underbrace{
    \Abs{\varepsilon_2\lambda^*_1-\varepsilon_1\lambda^*_2}
}_{\equiv\text{(iii)}}
\underbrace{
    \Abs{\frac{1}{q_1^2\lambda^*_1-q_2^2\lambda^*_2 - (q_1^2\varepsilon_1-q_2^2\varepsilon_2)}}
}_{\equiv\text{(iv)}}
\end{align*}
We now proceed to bound each factor:
\begin{itemize}
    \item For (i), the term $\left(q_1q_2\right)^2$ is bounded by $1$ since $q_1,q_2\in[0,1]$. 
    \item For (ii), the term $\Abs{\frac{q_1-q_2}{q_1^2\lambda^*_1-q_2^2\lambda^*_2}}$ is equal to $\left(\frac{\gamma}{\delta}\right)^{-1}$  by \eqref{eq:lv_one_channel_gamma_delta_estimator}.
    \item For (iii), 
from \cref{claim:lv_one_channel_equilibrium_bounds} we obtain the bound $0\le\lambda^*_i\le\frac{\beta\gamma}{4\alpha\delta}$, and therefore the term $\Abs{\varepsilon_2\lambda^*_1-\varepsilon_1\lambda^*_2}$ is bounded by $2\left(\frac{\beta\gamma}{4\alpha\delta}\right)\varepsilon=\frac{\beta\gamma}{2\alpha\delta}\varepsilon$. 
    \item For (iv), the term $\Abs{\frac{1}{q_1^2\lambda^*_1-q_2^2\lambda^*_2 - (q_1^2\varepsilon_1-q_2^2\varepsilon_2)}}$ is equal to:
    \begin{align*}
        \text{(iv)}
        &\equiv
        \Abs{\frac{1}{q_1^2\lambda^*_1-q_2^2\lambda^*_2 - (q_1^2\varepsilon_1-q_2^2\varepsilon_2)}}
        \\&=
        \frac{1}{\Abs{p_1-p_2}}\Abs{\frac{q_1^2\lambda^*_1-q_2^2\lambda^*_2 - (q_1^2\varepsilon_1-q_2^2\varepsilon_2)}{p_1-p_2}}^{-1}
        \\&=
        \frac{1}{\Abs{p_1-p_2}}
        \Abs{
        \underbrace{\frac{q_1^2\lambda^*_1-q_2^2\lambda^*_2}{p_1-p_2}}_{\text{\eqref{eq:lv_one_channel_gamma_delta_estimator}}}
        -
        \frac{q_1^2\varepsilon_1-q_2^2\varepsilon_2}{p_1-p_2}
        }^{-1}
        \\&=
        \frac{1}{\Abs{p_1-p_2}}
        \Abs{
        \frac{\gamma}{\delta}
        -
        \frac{q_1^2\varepsilon_1-q_2^2\varepsilon_2}{p_1-p_2}
        }^{-1}
    \end{align*}
    Note that
    $
    \Abs{\frac{q_1^2\varepsilon_1-q_2^2\varepsilon_2}{p_1-p_2}}
    \le
    \frac{2\varepsilon}{\Abs{p_1-p_2}}
    $. 
    When $\varepsilon$ is small enough, and specifically when the bound $\varepsilon\le\frac{\gamma}{\delta}\frac{\Abs{p_1-p_2}}{4}$ holds, we obtain:
    \begin{align*}
        \Abs{
        \frac{\gamma}{\delta}
        -
        \frac{q_1^2\varepsilon_1-q_2^2\varepsilon_2}{p_1-p_2}
        }^{-1}
        &\le
        \frac{\delta}{\gamma}
        \Abs{1-\frac{1}{2}}^{-1}
        \le
        2\frac{\delta}{\gamma}
    \end{align*}
    and therefore:
    \begin{align*}
        \text{(iv)}
        &\le
        \frac{2}{\Abs{p_1-p_2}}\frac{\delta}{\gamma}
    \end{align*}    
\end{itemize}
Aggregating results (i)-(iv) above, we obtain the overall bound:
\begin{align*}
\Abs{\frac{\alpha}{\beta} - \hat{\frac{\alpha}{\beta}}}
&=
\underbrace{
    \left(q_1q_2\right)^2
}_{\le1}
\underbrace{
    \Abs{\frac{q_1-q_2}{q_1^2\lambda^*_1-q_2^2\lambda^*_2}}
}_{=\frac{\delta}{\gamma}}
\underbrace{
    \Abs{\varepsilon_2\lambda^*_1-\varepsilon_1\lambda^*_2}
}_{\le\frac{\beta\gamma}{2\alpha\delta}\varepsilon}
\underbrace{
    \Abs{\frac{1}{q_1^2\lambda^*_1-q_2^2\lambda^*_2 - (q_1^2\varepsilon_1-q_2^2\varepsilon_2)}}
}_{\le \frac{2}{\Abs{p_1-p_2}}\frac{\delta}{\gamma}}
\\
&\le
\frac{\varepsilon}{\Abs{p_1-p_2}}\frac{\beta\delta}{\alpha\gamma}
\end{align*}
\end{proof}

\begin{proposition}[Cost of $\alpha/\beta$ estimation error]
\label{claim:single_channel_estimation_price_bound}
Let $\frac{\alpha}{\beta}$ be the engagement ratio parameter of a one-channel Lotka-Volterra system, and let $\hat{\left(\frac{\alpha}{\beta}\right)}$ be an estimate of these parameters. Let $\lambda^*_\opt$ be the engagement rate of the optimal static policy, and denote $\lambda^*(x)=\lambda^*\left(\hat{p}(x)\right)$. When $\Abs{\frac{\alpha}{\beta}- \hat{\left(\frac{\alpha}{\beta}\right)}}\le\min\Set{\frac{\alpha}{2\beta},1}$ The price of estimation error is bounded by:
\begin{equation*}
    \lambda^*_\opt-\lambda^*\left(\hat{\left(\frac{\alpha}{\beta}\right)}\right) \le 
    \left(\frac{\gamma}{\delta}\right)
    \min\Set{
        \left(2\frac{\alpha}{\beta}\right)^{-2}\Abs{\frac{\alpha}{\beta} - \hat{\left(\frac{\alpha}{\beta}\right)}},
        \left(4\frac{\alpha}{\beta}\right)^{-1}
    }
\end{equation*}
\end{proposition}
\begin{proof}
Denote $r=\frac{\alpha}{\beta}, x=\hat{\left(\frac{\alpha}{\beta}\right)}$, and assume without loss of generality that $\frac{\gamma}{\delta}=1$ and $r \le 1$. The optimal equilibrium engagement rate is given by:
\begin{equation*}
    \lambda^*_\opt = \begin{cases}
    \frac{1}{4r} & r\in\left(0,\frac{1}{2}\right] \\
    1-r & r\in\left(\frac{1}{2},1\right]
    \end{cases}
\end{equation*}
The chosen policy $\hat{p}(x)$ is given by:
\begin{equation*}
    \hat{p}(x) = \begin{cases}
    1-2x & x\in\left[0,\frac{1}{2}\right] \\
    0 & \mathrm{otherwise}
    \end{cases}
\end{equation*}
Assume without loss of generality that $x\in\left[0,\frac{1}{2}\right]$, as values of $x$ outside the interval can be clipped to its edges without affecting the result. The equilibrium engagement rate of the selected policy is given by:
\begin{equation*}
    \lambda^*(x) = \lambda^*\left(\hat{p}(x)\right) =  \begin{cases}
    0 & x\in\left[0,\frac{r}{2}\right]
    \\
    \frac{1}{2x}\left(1-\frac{r}{2x}\right) & x\in\left(\frac{r}{2},\frac{1}{2}\right]
    \end{cases}
\end{equation*}
Denote $\Delta(x)=\lambda^*_\opt-\lambda^*(x)$. We obtain:
\begin{equation*}
    \Delta(x)
    =\lambda^*_\opt-\lambda^*(x)
    =\begin{cases}
    \frac{1}{4r}
     & r\in\left(0,\frac{1}{2}\right], x\in\left[0,\frac{r}{2}\right] \\
    \frac{\left(x-r\right)^2}{4x^2r}
     & r\in\left(0,\frac{1}{2}\right], x\in\left(\frac{r}{2},\frac{1}{2}\right] \\
    (1-r)
     & r\in\left(\frac{1}{2},1\right], x\in\left[0,\frac{r}{2}\right] \\
    (1-r) - \frac{1}{2x}\left(1-\frac{r}{2x}\right)
     & r\in\left(\frac{1}{2},1\right], x\in\left(\frac{r}{2},\frac{1}{2}\right] \\
    \end{cases}
\end{equation*}
Observe that $\frac{1}{4r}\ge1-r$ for all $r\in(0,1]$, and therefore we obtain for all $x,r$:
\begin{equation}
    \Delta(x)\le\frac{1}{4r}
\end{equation}
From the convexity of $\Delta(x)$ in the region around $x=r$ we obtain:
\begin{equation}
    \Delta(x)\le\frac{1}{2r^2}\Abs{x-r}
\end{equation}
Finally, combining the two bounds yields the final result. A geometric interpretation of this claim is illustrated in \cref{fig:single_channel_estimation_price_bound}.
\end{proof}

\begin{figure}
    \centering
    \includegraphics[width=\linewidth]{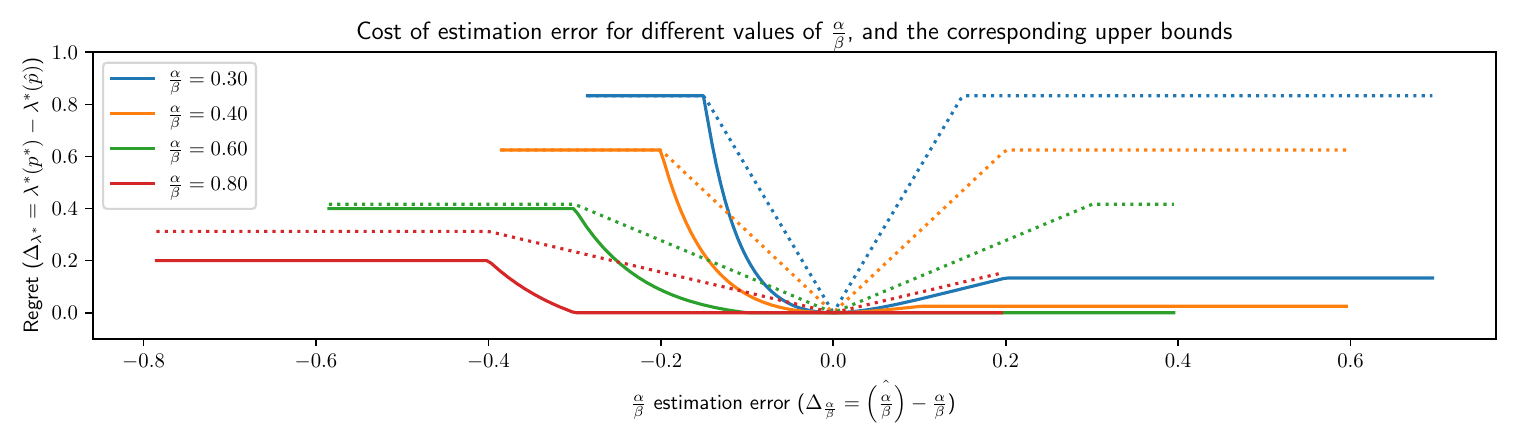}
    \caption{Graphical illustration of \cref{claim:single_channel_estimation_price_bound}. Cost of estimation error for different values of $\frac{\alpha}{\beta}$, and their corresponding upper bounds given by the claim.}
    \label{fig:single_channel_estimation_price_bound}
\end{figure}

\subsection{Optimal Stationary Policy From Engagement Predictions}
\label{subsection:single_channel_optimal_policy_proofs}
\begin{definition}[Expected observable rate]
Let $u\in\Users$, $p\in[0,1]$, and $T>0$. Let $p\in[0,1]$, denote the corresponding static policy by $\pi_p$. The expected observable rate $\bar{\lambda}_u(p;T)$ is defined as:
$$
\bar{\lambda}_u(p;u) = \expect{\pi}{\frac{1}{T}\Size{\TPP_{\pi_p}(u;T)}}
$$
where expectation is taken over the stochastic decisions of $\pi_p$.
\end{definition}

\begin{definition}[Lokta-Volterra approximation of $\TPP$]
Let $u\in\Users$, and $T>0$. Denote by $p^*$ the maximizer of expected observable rate:
$$
p^* = \argmax_{p\in[0,1]} \bar{\lambda}_u(p;u)
$$

The LV approximation of $\TPP(u;T)$ is defined as:
$$
\theta^*_u = 
\argmin_{\theta} \max_{p\in[0,1]} {\Abs{\bar{\lambda}_u(p;u)-\lambda^*(p;\theta)}}
$$
such that $\argmax_p \lambda^*(p;\theta) = p^*$. The corresponding approximation error is defined as:
$$
\varepsilon_{\mathrm{LV},u} = 
\max_{p\in[0,1]} {\Abs{\bar{\lambda}_u(p;u)-\lambda^*(p;\theta^*_u)}}
$$
\end{definition}

\paragraph{Notations.}
When $u$ is clear from the context, we denote $\theta^*=\theta^*_u$, $\varepsilon_{\mathrm{LV}}=\varepsilon_{\mathrm{LV},u}$. We use $\alpha^*, \beta^*, \dots$ to refer to the corresponding parts of the Lokta-Volterra parameters vector $\theta^*$.

We are now ready to state and prove the main theorem for this section:
\begin{theorem}[Regret bound for learned static policy. Formal version of \cref{thm:bound}]
\label{thm:optimal_static_policy_can_be_learned}
Let $p_1, p_2 \in [0,1]$ denote two static forced-break policies, and denote by $\Users$ the set of users, and assume they remain engaged under the stationary policies $\pi(p_1)$ and $\pi(p_2)$. Assume $S_u(p;T)\sim\TPP_{\pi_p\circ\psi}(u;T)$, and let 
$\mu = \left(\max_{u\in\Users} \frac{\bar{\gamma}_u}{\bar{\delta}_u}\right)\cdot\left(\max_{u'\in\Users} \frac{\bar{\delta}_{u'
}}{\bar{\gamma}_{u'}}\right)$,
$\nu = \max_{u\in\Users} \left(\frac{\bar{\beta}_u}{\bar{\alpha}_u}\right)$.

Let $f_{p_1},f_{p_2}:\Users\to\Reals_+$ be functions predicting $\EmpiricalRateP{p_1}, \EmpiricalRateP{p_2}$, respectively. 
Denote the learned policy by $\hat{p}$, and the optimal policy by $p^*$. 

If (i) the expected RMSE of $f_{p_1},f_{p_2}$ is bounded by $\varepsilon_\mathrm{pred}$, (ii) the average absolute deviation of $\frac{1}{T}\Size{\TPP(u;T)}$ is bounded by $\varepsilon_\mathrm{dev}$, and (iii) the expected LV approximation error of the system is bounded by $\varepsilon_\mathrm{LV}$, then the learned policy $\hat{p}$ has bounded regret:
\begin{equation*}
    \expect{u, \pi}{\Abs{\EmpiricalRateP{p^*}-\EmpiricalRateP{\hat{p}}}}
    \le 
    \frac{\eta_\TPP}{\Abs{p_1-p_2}} (\varepsilon_\mathrm{pred}+\varepsilon_\mathrm{dev}+\varepsilon_\mathrm{LV})
\end{equation*}
where expectation is taken over stochastic choices of policies, and $\eta_\TPP=g(\mu,\nu)\in\mathrm{poly}(\mu, \nu)$.
\end{theorem}
\begin{proof}
By assumption (i), the functions $f_{p_1}, f_{p_2}$ have bounded expected RMSE:
\begin{equation}
    \label{eq:f_pi_expected_pac_square_error}
    \expect{u}{\left(f_{p_i}(u) - \EmpiricalRateP{p_i}\right)^2} \le \varepsilon_\mathrm{pred}^2
\end{equation}
Applying Jensen's inequality with the convex function $\varphi(x)=x^2$ yields:
\begin{equation*}
    \left(\expect{u}{\Abs{f_{p_i}(u) - \EmpiricalRateP{p_i}}} \right)^2
    \le \expect{u}{\left(f_{p_i}(u) - \EmpiricalRateP{p_i}\right)^2}
\end{equation*}
Combining with \eqref{eq:f_pi_expected_pac_square_error} and taking the square root, we obtain an upper bound on the expected absolute error:
\begin{equation}
    \label{eq:f_pi_expected_pac_absolute_error}
    \expect{u}{\Abs{f_{p_i}(u) - \EmpiricalRateP{p_i}}}
    \le \varepsilon_\mathrm{pred}
\end{equation}

Let $\Delta_f=\Abs{f_{p_i}(u) - \lambda^*(p_i)}$ apply the triangle inequality to obtain:
\begin{align*}
    \Delta_f
    &=\Abs{f_{p_i}(u) - \lambda^*(p_i)}
    \\
    &\le
    \Abs{f_{p_i}(u) - \EmpiricalRateP{u}}
    + \Abs{\EmpiricalRateP{u} - \bar{\lambda}(p_i; u)}
    + \Abs{\bar{\lambda}(p_i; u) - \lambda^*(p_i)}
\end{align*}

Denote $\varepsilon_f=\varepsilon_\mathrm{pred}+\varepsilon_\mathrm{dev}+\varepsilon_\mathrm{LV}$. 
Applying the triangle inequality and using the bounds in \eqref{eq:f_pi_expected_pac_absolute_error} together with assumptions (ii), (iii), we obtain:
\begin{align}
    \expect{u,\pi}{\Delta_f} 
    \le&
    \expect{u}{\Abs{f_{p_i}(u) - \EmpiricalRateP{u}}}
    \nonumber \\
    &+ \expect{u,\pi}{\Abs{\EmpiricalRateP{u}-\bar{\lambda}(p_i; u)}}
    \nonumber \\ 
    &+ \expect{u}{\Abs{\bar{\lambda}(p_i; u)-\lambda^*(p_i; \theta^*_u)}}
    \nonumber \\ 
    \label{eq:total_alpha_beta_estimation_error}
    \le&
    \varepsilon_\mathrm{pred}
    +
    \varepsilon_\mathrm{dev}
    +
    \varepsilon_\mathrm{LV}
    =
    \varepsilon_f
\end{align}
Denote $\theta^*_u=(\alpha,\beta,\gamma,\delta)$. The empirical value $\hat{\left(\frac{\alpha}{\beta}\right)}$ of $\left(\frac{\alpha}{\beta}\right)$ is given by \cref{definition:single_channel_alpha_beta_estimator}. Denote the estimation error by
$\Delta_{\frac{\alpha}{\beta}}=\Abs{
    \hat{\left(\frac{\alpha}{\beta}\right)}
    -{\left(\frac{\alpha}{\beta}\right)}
}$.

By \cref{claim:single_channel_alpha_beta_estimation_error_bound}, the following pointwise upper bound on $\Delta_{\frac{\alpha}{\beta}}$ applies when $\Delta_f\le\frac{\gamma}{\delta}\frac{\Abs{p_1-p_2}}{4}$:
\begin{equation}
    \label{eq:alpha_beta_error_small_epsilon_upper_bound}
    \Delta_{\frac{\alpha}{\beta}}
    \le\frac{\Delta_f}{\Abs{p_1-p_2}}\frac{\beta\delta}{\alpha\gamma}
\end{equation}
Plugging in the bound on the expected value of $\Delta_f$ into \eqref{eq:alpha_beta_error_small_epsilon_upper_bound}, we obtain in expectation:
\begin{align}
    \expect{u,\pi}{\Delta_{\frac{\alpha}{\beta}} \mid \Delta_f\le\frac{\gamma}{\delta}\frac{\Abs{p_1-p_2}}{4}}
    &\le\expect{u,\pi}{\frac{\Delta_f}{\Abs{p_1-p_2}} \frac{\beta\delta}{\alpha\gamma} \mid \Delta_f\le\frac{\gamma}{\delta}\frac{\Abs{p_1-p_2}}{4}}
    \nonumber \\
    \label{eq:alpha_beta_error_small_epsilon_upper_bound_expectation}
    &\le\frac{\varepsilon_f}{\Abs{p_1-p_2}}\max_u \frac{\beta\delta}{\alpha\gamma} 
\end{align}
Next, we apply \cref{claim:single_channel_estimation_price_bound}. Denote $\Delta_{\lambda^*}=\lambda^*(p^*)-\lambda^*(\hat{p})$, and define the following probability event:
\begin{equation*}
    A = 
    \left(\Delta_f\le\frac{\gamma}{\delta}\frac{\Abs{p_1-p_2}}{4}\right)
    \text{ and }
    \left(\Delta_\frac{\alpha}{\beta}\le\frac{1}{2\nu}\right)
\end{equation*}
Note that the bound in \cref{claim:single_channel_estimation_price_bound} is represented as a minimum between two functions, one linear in $\varepsilon$ and one constant. To leverage this property, apply the law of total expectation:
\begin{equation}
    \label{eq:delta_lambda_star_bound_total_expectation}
    \expect{u,\pi}{\Delta_{\lambda^*}}
    =
    \expect{u,\pi}{\Delta_{\lambda^*} \mid A} \prob{}{A}
    + \expect{u,\pi}{\Delta_{\lambda^*} \mid \bar{A}} \prob{}{\bar{A}}
\end{equation}

Under $A$, the first term in \eqref{eq:delta_lambda_star_bound_total_expectation} can be bounded by the linear term in \cref{claim:single_channel_estimation_price_bound}. Taking $\prob{}{A}\le1$ and combining with equation \eqref{eq:alpha_beta_error_small_epsilon_upper_bound}:
\begin{align}
    \expect{u,\pi}{\Delta_{\lambda^*} \mid A}
    \prob{}{A}
    &\le 
    \expect{u,\pi}{\Delta_{\lambda^*} \mid A}
    \nonumber \\
    &\le
    \expect{u,\pi}{\frac{\beta^2\gamma}{2\alpha^2\delta}\Delta_{\frac{\alpha}{\beta}} \mid A}
    \nonumber \\
    &\le
    \expect{u,\pi}{\frac{\beta^2\gamma}{2\alpha^2\delta} \frac{\Delta_f}{\Abs{p_1-p_2}}\frac{\beta\delta}{\alpha\gamma} \mid A}
    \nonumber \\
    \label{eq:delta_lambda_star_upper_bound_linear_term}
    &\le
    \frac{\nu^3}{2\Abs{p_1-p_2}}\varepsilon_f
\end{align}

The expectation factor in the second term of \eqref{eq:delta_lambda_star_bound_total_expectation} can be bounded by the constant term in \cref{claim:single_channel_estimation_price_bound}:
\begin{equation}
    \label{eq:expected_cost_of_estimation_error_large_deviation_upper_bound}
    \expect{u,\pi}{\Delta_{\lambda^*} \mid \bar{A}}
    \le 
    \frac{1}{4} \max_u \frac{\beta\gamma}{\alpha\delta} \le \frac{\nu}{4} \max_u \frac{\gamma}{\delta}
\end{equation}
Decompose the probability factor $\prob{}{\bar{A}}$ using the law of total probability:
\begin{align*}
    \prob{}{\bar{A}} 
    &= 
    \prob{}{\Delta_f>\frac{\gamma}{\delta}\frac{\Abs{p_1-p_2}}{4}}
    +
    \prob{}{\Delta_\frac{\alpha}{\beta}>\frac{1}{2\nu} \mid \Delta_f\le\frac{\gamma}{\delta}\frac{\Abs{p_1-p_2}}{4}} \prob{}{ \Delta_f\le\frac{\gamma}{\delta}\frac{\Abs{p_1-p_2}}{4}}
    \\
    &\le
    \prob{}{\Delta_f>\frac{\gamma}{\delta}\frac{\Abs{p_1-p_2}}{4}}
    +
    \prob{}{\Delta_\frac{\alpha}{\beta}>\frac{1}{2\nu} \mid \Delta_f\le\frac{\gamma}{\delta}\frac{\Abs{p_1-p_2}}{4}}
\end{align*}
Apply Markov's inequality $\prob{}{\Abs{X}\ge a}\le \frac{\expect{}{\Abs{X}}}{a}$ on the probabilities to obtain:
\begin{align}
    \prob{}{\Delta_f>\frac{\gamma}{\delta}\frac{\Abs{p_1-p_2}}{4}}
    &\le 
    \expect{u,\pi}{\Delta_f} \left(\frac{\gamma}{\delta}\frac{\Abs{p_1-p_2}}{4}\right)^{-1}
    \nonumber \\
    \label{eq:large_delta_f_probability_bound}
    &\underset{\text{by \eqref{eq:total_alpha_beta_estimation_error}}}{\le}
    \varepsilon_f \frac{4}{\Abs{p_1-p_2}} \max_u \frac{\delta}{\gamma}
    \\
    \nonumber \\
    \prob{}{\Delta_\frac{\alpha}{\beta}>\frac{1}{2\nu} \mid \Delta_f\le\frac{\gamma}{\delta}\frac{\Abs{p_1-p_2}}{4}}
    &\le 
    \expect{u,\pi}{\Delta_\frac{\alpha}{\beta} \mid \Delta_f\le\frac{\gamma}{\delta}\frac{\Abs{p_1-p_2}}{4}}
    \nonumber \\
    &\underset{\text{by \eqref{eq:alpha_beta_error_small_epsilon_upper_bound_expectation}}}{\le}
    \frac{\varepsilon_f}{\Abs{p_1-p_2}}\max_u \frac{\beta\delta}{\alpha\gamma}
    \nonumber \\
    \label{eq:large_delta_alpha_beta_probability_bound}
    &\le
    \frac{\varepsilon_f}{\Abs{p_1-p_2}}\nu \max_u\frac{\delta}{\gamma}
\end{align}
Plugging back equations \eqref{eq:delta_lambda_star_upper_bound_linear_term}, \eqref{eq:expected_cost_of_estimation_error_large_deviation_upper_bound},\eqref{eq:large_delta_f_probability_bound},\eqref{eq:large_delta_alpha_beta_probability_bound} into equation \eqref{eq:delta_lambda_star_bound_total_expectation}, we obtain bounds for each term:

\begin{equation}
    \expect{u,\pi}{\Delta_{\lambda^*}}
    =
    \underbrace{\expect{u,\pi}{\Delta_{\lambda^*} \mid A} \prob{}{A}}_{\text{by eq. \ref{eq:delta_lambda_star_upper_bound_linear_term}}}
    + 
    \underbrace{\expect{u,\pi}{\Delta_{\lambda^*} \mid \bar{A}}}_{\text{by eq. \ref{eq:expected_cost_of_estimation_error_large_deviation_upper_bound}}}
    \underbrace{\prob{}{\bar{A}}}_{\text{by eqs. \ref{eq:large_delta_f_probability_bound},\ref{eq:large_delta_alpha_beta_probability_bound}}}
\end{equation}

we obtain:
\begin{equation*}
    \expect{u,\pi}{\Delta_{\lambda^*}}
    \le
    \frac{\varepsilon_f}{\Abs{p_1-p_2}} \left(
    \frac{\nu^3}{2}
    + \left(\nu+\frac{\nu^2}{4}\right)\mu
    \right) = \varepsilon_{\lambda^*}
\end{equation*}
To obtain the regret bound on the empirical rates, we apply assumptions (ii), (iii) once again to bound the expected difference between $\lambda^*(p)$ and $\EmpiricalRateP{p}$, and apply the triangle inequality:
\begin{align*}
    \expect{u,\pi}{\Abs{\EmpiricalRateP{p^*}-\EmpiricalRateP{\hat{p}}}}
    &\le 
    \varepsilon_{\lambda^*} + 2(\varepsilon_\mathrm{dev}+\varepsilon_\mathrm{LV})
\end{align*}
Note that $\frac{\nu}{\Abs{p_1-p_2}}>1$, as $\frac{\beta}{\alpha}\ge 1$ since all the users are assumed to remain engaged in the long term, and $\Abs{p_1-p_2}\le 1$ as $p_1,p_2\in[0,1]$. Therefore, the function 
$
\eta_\TPP = g(\mu, \nu)=\left(
    \frac{\nu^3}{2}
    + \left(\nu+\frac{\nu^2}{4}\right)\eta
    + 2\nu
\right)
$ satisfies:
\begin{equation*}
    \expect{u, \pi}{\Abs{\EmpiricalRateP{p^*}-\EmpiricalRateP{\hat{p}}}}
    \le 
    \frac{\eta_\TPP}{\Abs{p_1-p_2}} (\varepsilon_\mathrm{pred}+\varepsilon_\mathrm{dev}+\varepsilon_\mathrm{LV})
\end{equation*}
\end{proof}
\section{Experimental details}
\label{app:experimental_details}

\subsection{Data}

\paragraph{MovieLens-1M}
We base our main experimental environment on the MovieLens-1M dataset,
which is a standard benchmark dataset used widely in recommendation system research \citep{harper2015movielens}.
The dataset includes \LvExperimentCfNRatings{} ratings provided by \LvExperimentCfNUsers{} users and for \LvExperimentCfNItems{} movies.
Ratings are in the range $\{1,\dots,5\}$,
and all users in the dataset have at least 20 reported ratings.
The dataset is publicly available at:
\url{https://grouplens.org/datasets/movielens/1m/}.

\paragraph{Goodreads.}
We validate our results using the Goodreads book recommendations dataset, which is a common benchmark dataset used in recommendation systems research \citep{wan2018item, wan2019fine}. 
Ratings are in the range $\{1,\dots,5\}$,
and the dataset is filtered to only include users with at least 20 reported ratings.
We use the official comic-books genre subset of the dataset,
which includes \LvExperimentCfNRatingsGoodreads{} ratings provided by \LvExperimentCfNUsersGoodreads{} users and for \LvExperimentCfNItemsGoodreads{} books after pre-processing.
The dataset is publicly available at:
\url{https://sites.google.com/eng.ucsd.edu/ucsdbookgraph/}.
Pre-processing code is available in the repository: \LvCodeURL{}. 
We follow an identical experimental procedure for all datasets.

\paragraph{Data partitioning.}
To learn latent user and item features, \LvExperimentCfTrainingSetPct{}\% of all ratings were drawn at random. Stratified sampling was applied to ensure that all users and items were covered, and so that each users have roughly the same proportion of ratings used for this step. These ratings were only used only for learning a CF model, and were discarded afterwards.
The remaining \LvExperimentCfTestSetPct{}\% data points were used for training and testing.
For these, we first randomly sampled \LvExperimentEvaluationSetUsers{} users to form the test set.
Then, the remaining users were partitioned into the main train set $\S$,
which included \LvExperimentControlGroupWeightPct{}\% ($\approx$\LvExperimentControlGroupSize{} for ML1M, $\approx$\LvExperimentControlGroupSizeGoodreads{} for Goodreads) of these users,
and the experimental treatment sets $D^{(j)}$,
each including \LvExperimentTreatmentGroupWeightPct{}\% ($\approx$\LvExperimentTreatmentGroupSize{} for ML1M, $\approx$\LvExperimentTreatmentGroupSizeGoodreads{} for Goodreads) users for $N=3$.
This procedure was repeated \LvExperimentNumRepetitions{} times with different random seeds.
We report average results, together with \LvExperimentConfidenceLevelPct{}\% t-distribution confidence intervals representing variation between runs.

\subsection{Implementation Details}

\begin{itemize}
\item
\textbf{Hardware}: All experiments were run on a single laptop, with 16GB of RAM, M1 Pro processor, and with no GPU support.

\item 
\textbf{Runtime}: A single run consisting the entire pipeline (data loading and partitioning, collaborative filtering, training classifiers, simulating dynamics, learning policies, measuring and comparing performance) takes roughly \LvExperimentTotalSimulationTimeMinutes{} minutes. The main bottleneck is the discrete LV simulation, taking roughly 70\% of runtime to compute,
mostly due to bookkeeping necessary for the non-stationary baselines.
Simulation code was optimized using the \textsc{numba} jit compiler,
which improves runtime.

\item
\textbf{Optimization packages}: 
\begin{itemize}
    \item \textbf{Collaborative filtering (CF)}: We use the \textsc{surprise} package \citep{Hug2020}, which includes an implementation of the SVD algorithm for CF. All parameters were set to default values.
    \item \textbf{Regression}: We use the \textsc{scikit-learn} implementation of linear regression for predicting long-term engagement from user features (i.e the prediction models $f_j(u)$ in \eqref{eq:theta_empirical_fit}). All parameters were set to default values. 
    \item \textbf{Non-Negative Least Squares (NNLS)}: We use the \textsc{scipy.optimize} implementation of NNLS. The algorithm was used with its default parameters.
\end{itemize}

\item 
\textbf{Code}: Code for reproducing all of our figures and experiments is available in the following repository:\\
\LvCodeURL{}.

\end{itemize}

\subsection{Other Baselines}
\begin{itemize}
\item \textbf{Safety}: In each step of the TPP simulation, look $k$ step back, and calculate the empirical rate $\tilde{\lambda}_i = \frac{k}{t_i-t_{i-k}}$. If this rate exceeds the threshold $\tilde{\lambda}_i>\tau$, the policy enters a `cool-down' policy state, serving only forced breaks until the next time period. In our experiments, we used thresholds $\tau\in\Set{14,16}$, $k=10$ look-behind steps, and defined the cool-down period as $0.5$ time units.

\item \textbf{Oracle}: To estimate the effect of perfect predictions, we implement an oracle predictor $f^{\mathrm{oracle}}_p(u)$ which has access to the latent user parameters. For a given $u$ and for each $p$, the predictor outputs the infinite-horizon LV equilibrium for $u$, namely
$f^{\mathrm{oracle}}_p(u)=\lambda^*(p;\tilde{\theta}_u)$. 
We define
$\tilde{\theta}_u=(\alpha_u,\tilde{\beta}_u,\gamma_u,\delta_u)$, where $\alpha_u,\gamma_u, \delta_u$ are the unobserved parameters for the given user, and $\tilde{\beta}_u$ is the expected value of $\beta_{ux}$ induced by the distribution over recommended items $x$ induced by the recommendation policy $\psi$.
We view $\tilde{\theta}_u$ as a useful proxy for the otherwise unattainable $\bar{\theta}_u$.
\end{itemize}

\subsection{Hyperparameters}
\label{subsec:experiment_hyperparameters}
\begin{itemize}
    \item
    \textbf{Collaborative filtering}:
    We used $d=\LvExperimentCfNFactors$ latent factors and enabled bias terms, which ensured performance is close to the benchmark of $\mathrm{RMSE}=0.873$ reported in the \textsc{surprise} documentation. We used the vanilla SVD solver, with all hyperparameters set to their default values.

    \item
    \textbf{Recommendation policy}: Softmax temperature was set to \LvExperimentSoftmaxT{}.

    \item
    \textbf{Prediction}: 
    We trained regressors $f(u)$ on input feature vectors consisting of three components: 
    \begin{equation}
    \label{eq:user_feature_vector}
    u = (\tilde{v}_u, b_u, \hat{\mu}_u) \in \R^{d+2}
    \end{equation}
    The three components are: 
    (i) SVD latent user factors $\tilde{v}_u\in\Reals^d$, 
    (ii) SVD user bias term $b_u\in\Reals$,
    (iii) an additional feature consisting of the average predicted ratings for unseen items $\hat{r}_u$ weighted by recommendation probability, which we found to slightly improve predictive performance:
    \begin{equation}
    \label{eq:predicted_ratings_feature}
    \hat{\mu}_u=\sum_{x\in\mathrm{holdout}(u)}\hat{r}_{ux} \cdot \softmax_x(\hat{r}_u)
    \end{equation}
    Where $\mathrm{holdout}(u)$ is the set of unseen items corresponding to user $u$, $\hat{r}_{ux}\in[1,5]$ are the predicted ratings, and $\hat{r}_u\in\left[1,5\right]^\Size{\mathrm{holdout}(u)}$ is the vector of all predicted ratings used for softmax recommendation as described in \cref{sec:experiments}.
    We chose to focus on linear models since the treatment datasets are relatively small
    (each $|\S^{(j)}| \approx 500$), and since other model classes (including boosted trees and MLPs) did not perform significantly better.

    \item 
    \textbf{Engagement dynamics}:
    Interaction sequences for each user were generated according to 
    the interaction dynamics described in \cref{sec:engagement_dynamics}. 
    We denote this process by
    $\LVTPP(p;u)$, 
    and describe it in detail in the next section. 
    Latent sates were initialized randomly with relative uniform noise around the theoretical LV equilibrium point $(\lambda_0, q_0)=((1+\xi_\lambda)\lambda^*, (1+\xi_q)q^*)$, where $\xi_\lambda, \xi_q\sim \mathrm{Uniform}(-0.1,0.1)$. Latent states were updated each $B=10$ recommendations to stabilize noise (see \cref{fig:discrete_and_continuous_lv_simulation}).
    When $x$ is recommended to $u$ at time $t$, latent states and $\Delta t$ are set according
    to $\beta_u(t)$, which depends on ratings $r_{ux}$ (true or mixed with predictions $u^\top x$ via $\kappa$).
    Specifically, we use $\beta_u(t) = r_{ux}^2/100 \in \Set{\LvExperimentLvBeta{}}$,
    which is convex,
    to accentuate the role of low ratings since they are underrepresented in the data.
    For $B \ge 1$, we take the effective $\beta_u(t)$ to be the average over the $B$ items recommended in that step.
    We set $\alpha=\LvExperimentLvAlpha$, and chose $\gamma=\LvExperimentLvGamma,\delta=\LvExperimentLvDelta$ (which together determine scale)
    so that typical values for engagement rate $\frac{1}{T}|S_u|$ are on the order of $\approx 10$ for the chosen $T=\LvExperimentSimulationLength$.

\end{itemize}

\begin{figure}
    \centering
    \includegraphics[width=\linewidth]{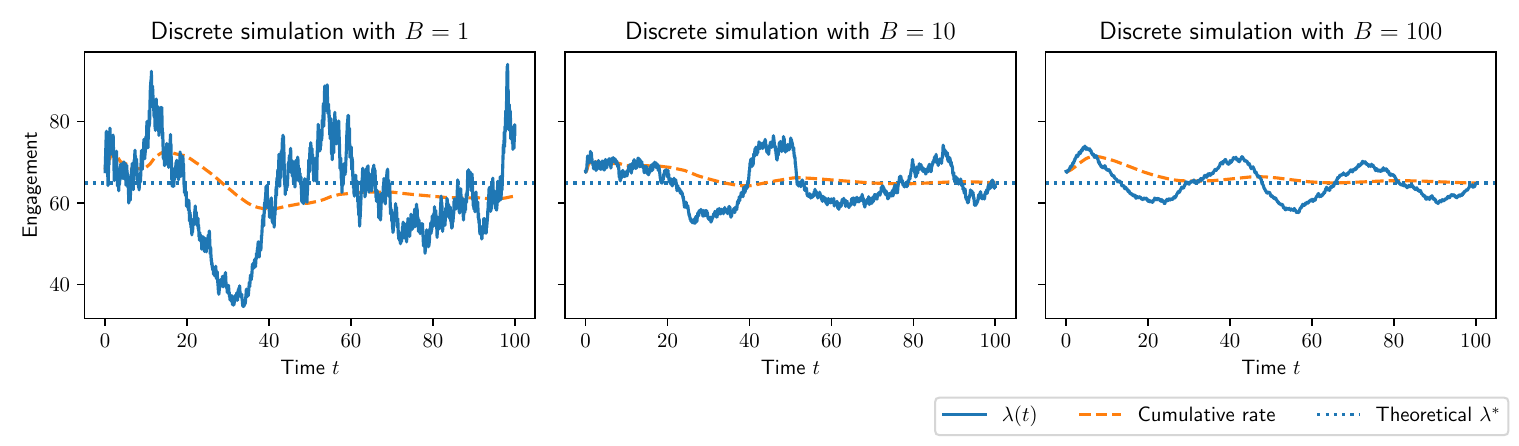}
    \caption{
    Example discrete sequence $S_u \sim \LVTPP(p;u)$, for varying batch sizes.
    $\LVTPP$ captures the general properties of our proposed behavioral model:
    note how cumulative averaging behavior (orange dashes) exhibits `habit formation',
    which our equilibrium approach targets (blue dots).
    For the same initial conditions $\lambda(0), z(0)$,
    the figure shows how varying the number of recommended items per step ($B$)
    `smooths' the discrete behavior (left: $B=1$, center: $B=10$, right: $B=100$).
    As $B$ is increased,
    $\LVTPP$ sequences approach a continuous LV trajectory;
    in general, and particularly when $\beta_u(t)$ varies by step and per recommended items---this is not the case. 
    }
    \label{fig:discrete_and_continuous_lv_simulation}
\end{figure}

\subsection{Discrete TPP for Lokta-Volterra Simulation}
\label{app:discrete_behavioral_model}

The Temporal Point Process (TPP) we use for simulating user interaction sequences $S_u$ is based on a discretization of the LV system described in \eqref{eq:lv}, using the forward Euler method with variable step sizes. We denote this process by $\LVTPP(u;T)$, and present the sampling procedure in \cref{algorithm:discrete_tpp}.

Each user is associated with discrete latent states $\lambda_i,q_i$,
and parameters $\alpha_u,\gamma_u,\delta_u$. 
Initial states $\lambda_0,q_0$ are set randomly.
At each step, and in time $t_i$, the system recommends $x_i=x(t_i)$,
which triggers updates in latent states, and determines the next time of interaction $t_{i+1}$. As noted, these update depend on item-specific parameters $\beta_{u,x_i}$.

Under stationary policy $\pi(p)$, the system recommends an item with probability $(1-p)$, and suggests a break with probability $p$. The simulator considers $B$ recommendation opportunities at each step. For each $k\in\Set{1,\dots,B}$, denote by $I_k\in\Set{0,1}$ the break indicator, equal to $0$ when a break is recommended at the $k$-th slot in the batch. Denote by $x\sim\psi$ the item recommended by the underlying policy $\psi$, and by $\beta(x)$ the corresponding LV hyperparameter as defined above.

\begin{algorithm}
    \caption{Sample from $\LVTPP(p;u)$}
    \label{algorithm:discrete_tpp}
\begin{algorithmic}
\ENSURE $y = x^n$
    \REQUIRE {
        Break probability $p\in[0,1]$ \newline
        Stationary content recommendation policy $\pi_0$ \newline
        Lotka-Volterra parameters $\theta_u=\left(\alpha,\beta,\gamma,\delta\right)$ \newline
        Time horizon $T>0$
    }
    \ENSURE {
        Interaction sequence $S_u\sim\LVTPP(\psi(p)\circ\pi_0; \theta_u)$
    }
    \STATE $i \gets 0$
    \STATE $t_0 \gets 0$
    \STATE $S_u \gets \Set{}$
    \WHILE{$t_i<T$}
        \FORALL{$k \in \Set{1,\dots,B}$}
            \STATE $I_k\sim\mathrm{Bernoulli}(1-p)$
            \STATE $x_k \sim \psi$
            \STATE $\beta_k \gets \beta(r_{u,x_k})$
        \ENDFOR
        \STATE $\Delta t_i \gets \lambda_{i}^{-1}$
        \STATE $\lambda_{i+1}
        \gets \lambda_i \left(1-\alpha + \tfrac{\sum_{k=1}^{B}I_k \beta_k}{B} q_i \right)$
        \STATE $ q_{i+1}
        \gets
        q_i \left(\gamma(1-q_i) - \tfrac{\sum_{k=1}^{B}I_k \delta}{B} \lambda_i\right)$
        \STATE $ t_{i+1} \gets t_i + \Delta t_i$ 
        \STATE $S_u \gets S_u \cup \Set{\left(t_i, (x_1,\dots,x_B), (I_1,\dots,I_B)\right)}$ \;
        \STATE $ i \gets i+1$ \;
    \ENDWHILE
\end{algorithmic}
\end{algorithm}

\subsection{Adaptive Refinement Policy}
\label{subsec:additional_structure_ratings}
Here we provide implementation details for the adaptive refinement method presented in \cref{subsec:beyond_stationary}. 
A formal description of this method is given in \cref{algorithm:adaptive_policy}.

Informally, we model users as providing to the system feedback regarding item quality, namely true ratings $r_{ux}$, for some of the recommended items $x$ they consume. For simplicity, we assume users report ratings with probability $\rho\in[0,1]$, independently for each item, where we vary $\rho$ across experimental conditions. Fixing an \emph{adaptation time} $T_0\in[0,T]$ (a hyperparameter), the method applies the standard learned breaking policy $\hat{\pi}$ (\Eqref{eq:optimal_p}) in the time frame $t \le T_0$,
during which it also collects and stores user-reported ratings. Then, at time $t=T_0$,
the method updates the learning policy on the basis of the new ratings data, and applies this policy until the eventual time $T$.
In particular, ratings are used to
replace the predicted component $\hat{\rho}_u$ in the user feature vector (\Eqref{eq:predicted_ratings_feature}) with a statistical estimate $\bar{\mu}$ based on seen data.
Denote the set of ratings collected until time $t$ by $r_{[0,t]}\in\Set{1,\dots,5}^*$. Using this notation, the average rating at time $t$ is: 
$$\bar{\mu}_{(0,t)}=\tfrac{1}{\Size{r_{[0,t]}}}\sum_{r\in r_{[0,t]}} r$$

Using the these notations, the adapted feature vector corresponding to user $u$ at time $T_0$ is $u_\mathrm{refined}=\left(\tilde{v}_u,b_u,\bar{\mu}_{(0,T_0)}\right)$ (compare to \eqref{eq:user_feature_vector}). 
The refined breaking policy $p_u^\mathrm{refined}$ is calculated by \eqref{eq:learned_phat}, 
with LV parameters $\hat{\theta}_\mathrm{refined}$, estimated using \eqref{eq:theta_empirical_fit}
according to the refined prediction vector
$\f(u_\mathrm{refined})
=\left(f_1(u_\mathrm{refined}),\dots, f_N(u_\mathrm{refined})\right)
$.
We note that the method relies on existing engagement predictors $f_j(u)$ without any retraining.

\begin{algorithm}[b]
    \caption{Adaptive policy optimization using sparse rating signals}
    \label{algorithm:adaptive_policy}
\begin{algorithmic}[1]
    \REQUIRE{
        Initial break probability $p\in[0,1]$ \newline
        Time horizon $T>0$ \newline
        Adaptation time $T_0\in[0,T]$ \newline
        User feature vector $u=(\tilde{v},b,\hat{\rho})$
    }
    \STATE Collect ratings data $r_{[0,T_0]}$ with breaking policy $\pi(p)$ until time $T_0$
    \STATE Construct updated feature vector $u'=(\tilde{v},b,\bar{\rho}_{[0,T_0]})$ using the average rating $\bar{\rho}_{[0,T_0]}$
    \STATE Compute updated long-term engagement predictions $\f(u')$ \;
    \STATE Use the LV policy optimization method to obtain updated policy $p'=p^*\left(\f(u')\right)$ \;
    \STATE Use breaking policy $\pi(p')$ for the remaining time $t\in[T_0,T]$ \;
\end{algorithmic}
\end{algorithm}

\section{Additional Experimental Results}
\label{app:additional_empirical_evaluation}
In this section, we provide additional empirical evidence:
\begin{itemize}
    \item \cref{subsec:goodreads_analysis} replicates our main experimental results in \cref{sec:experiments} on an additional real dataset---the Goodreads dataset of user book reviews.
    \item In \cref{subsec:stateless_behavioral_model}, we evaluate our learning approach under the stateless engagement model presented in \eqref{eq:stateless_model}. The model is unrelated to Lotka-Volterra and does not promote breaks.
    The experiment demonstrates how our approach is ``safe'', in the sense that it does not recommend breaks needlessly,
    thus extending \cref{cor:phase_shift2} beyond the realizable case.
\end{itemize}

\subsection{Goodreads Evaluation Results}
\label{subsec:goodreads_analysis}
\cref{fig:goodreads_experiment} shows results for the Goodreads experiment, in the same format as \cref{fig:experimental_results}.
Results exhibit performance and trends that are qualitatively similar to the MovieLens experiment in \cref{sec:experiments}.
The LV policy optimization method achieved better performance on this dataset (right pane): 
The performance of the \method{LV} is closer to \method{oracle} (\LvExperimentOverallImprovementOverOracleGoodreads{}\% in Goodreads, compared to \LvExperimentOverallImprovementOverOracle{}\% in MovieLens), 
the gap between the \method{LV} and \method{LV-adaptive} is smaller (+\LvExperimentOverallImprovementOverAdaptiveNegGoodreads{}\% in Goodreads, compared to +\LvExperimentOverallImprovementOverAdaptiveNeg{}\% in MovieLens),
and the gap between the \method{LV} and \method{best-of} methods is larger (+\LvExperimentOverallImprovementOverArgmaxGoodreads{}\% in Goodreads, compared to +\LvExperimentOverallImprovementOverArgmax{}\% in MovieLens). 
Varying user types (center pane) shows less variation across values of $\kappa$, indicating that predictors achieve satisfactory performance even when the $\kappa$ is low and prediction is harder (see \cref{subsec:experimental_setup}).
Varying treatments (right pane) also coincides with the observations: LV policy optimization on the Goodreads dataset exhibits less performance degradation as $p_1\to 0$, suggesting that less data may be sufficient for optimization. We attribute these results to the larger size of the dataset (approximately 3.6M interactions, compared to 1M interactions in MovieLens), and to the possibility that a stronger structure may exist in the book recommendations compared to general movie recommendation.

\begin{figure}[t]
    \centering
    \includegraphics[width=\linewidth]{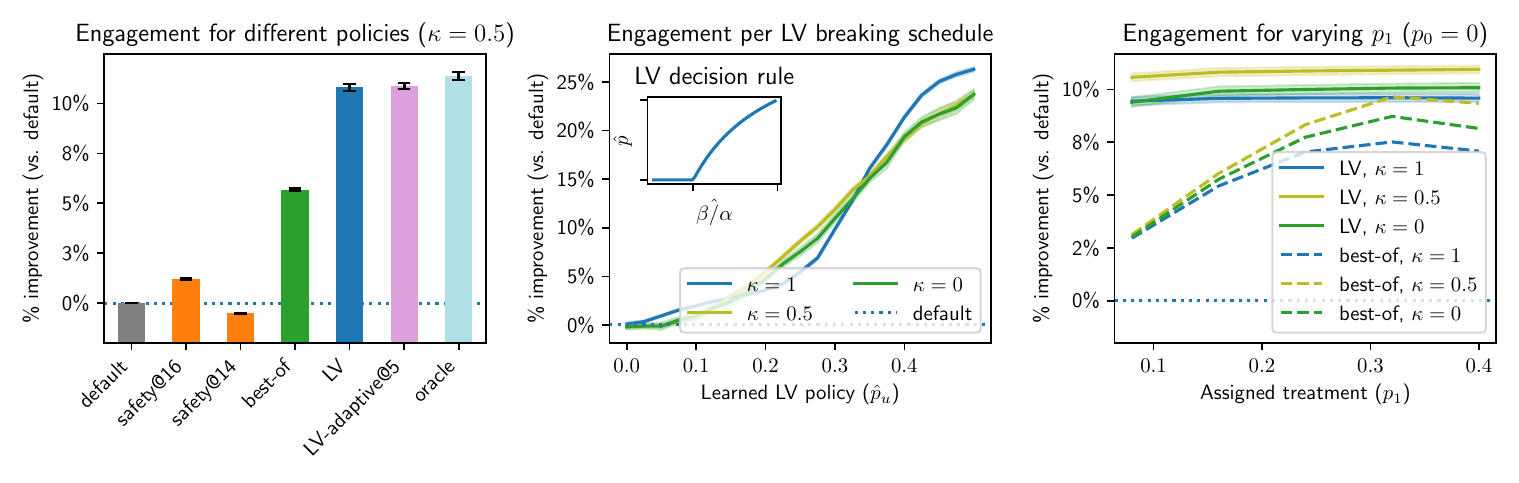}
    \caption{
    Results on the Goodreads comic-books dataset (compare to \autoref{fig:experimental_results}).
    \textbf{(Left)}
    Performance gain of different approaches
    (relative to \method{default} policy).
    \textbf{(Center)}
    Performance of \method{LV} by user group,
    partitioned by learned policies $\hat{p}_u$.
    \textbf{(Right)}
    Sensitivity to an increasingly aggressive experimental $p_1$
    ($N=2, p_0=0$).
    See \cref{subsec:goodreads_analysis} for analysis of resutls.
    }
    \label{fig:goodreads_experiment}
\end{figure}

\subsection{Distinct Behavioral Model}
\label{subsec:stateless_behavioral_model}
In this work, we propose the LV model as a behavioral hypothesis class for counterfactual prediction of long-term engagement.
As such, using the LV model within the learning-to-break optimization framework is a design choice, to be made at the discretion of the learner;
our LV model would be a good choice if it fits the data better than alternative model classes, given the amount of available data. This relation is made precise in our error bound (\cref{thm:bound}),
which bounds the error when using the LV model for \emph{any} underlying $\mathrm{TPP}$ (i.e., we make no assumptions about the true underlying data generation process).

Our main experiments evaluate our approach on data that is not LV dynamics, but nonetheless, bear some resemblance.
To complement these results, here we run an additional experiment in which
we empirically evaluate our LV-based approach on data that is generated by a behavioral model that is entirely distinct. 
In particular, we consider a user model in which consumption decisions only depend on the quality of recommended items, without any dependence on internal states---i.e., it is \emph{stateless}. This conforms to behavioral models which are implicitly assumed in conventional recommendation methods such as collaborative filtering. 

\begin{algorithm}[b!]
    \caption{Sample from $\StatelessTPP(p;u)$}
    \label{algorithm:stateless_tpp}
\begin{algorithmic}
    \REQUIRE{
        Break probability $p\in[0,1]$ \newline
        Stationary content recommendation probability $\psi$ \newline
        Scalar parameter $\theta_u=\tau>0$ \newline
        Time horizon $T>0$
    }
    \ENSURE{Interaction sequence $S_u\sim\StatelessTPP_{\pi(p)\circ\psi}(p;u)$}
    \STATE $i \gets 0$
    \STATE $t_0 \gets 0$
    \STATE $S_u \gets \Set{}$
    \WHILE{$t_i<T$}
        \FORALL{$k \in \Set{1,\dots,B}$}
            \STATE $I_k\sim\mathrm{Bernoulli}(1-p)$
            \STATE $x_k \sim \psi$
        \ENDFOR
        \STATE $ t_{i+1} \gets t_i + \frac{1}{\tau}\left(\frac{1}{B}\sum_{k=1}^{B} I_k r_{ux_k} \right)^{-1}$
        \STATE $S_u \gets S_u \cup \Set{\left(t_i, (x_1,\dots,x_B), (I_1,\dots,I_B)\right)}$
        \STATE $ i \gets i+1$
    \ENDWHILE
\end{algorithmic}
\end{algorithm}

\paragraph{Stateless content consumption.}
The data generation process is formalized in \cref{algorithm:stateless_tpp}.
As a means of capturing stateless behavior, we define a ``close-range'' temporal point process $\StatelessTPP(p;u)$ which generates sequences of user interactions based solely on the average rating of recommended items in a batch. Here we relate rating with utility, and assume that items having higher utility (and therefore higher ratings) induce more frequent interactions; in the same way, we consider break prompts as items having zero utility, and hence zero rating.
At time $t_i$, and given a batch of size $B$ with $k\le B$ recommended items and $B-k\ge 0$ breaks prompts, users acting according to the stateless behavioral will consume the next batch of content at time:
\begin{equation}
    \label{eq:stateless_behavioral_model}
    t_{i+1}=t_i+\frac{1}{\tau}\left(\frac{1}{B}\sum_{j=1}^k r_{u{x_j}}\right)^{-1}
\end{equation}
where $\left(r_{u{x_1}},\dots,r_{u{x_k}}\right)\in\Set{1,\dots,5}^k$ are the ratings of the items recommended at time $k$, and $\tau > 0$ is a constant latent parameter to be learned from data. The breaking policy $\pi$ decides on the number of items $k$ to recommend on each step. 
Since $r_{ux}\ge 1$ for all user-item pairs, the time difference $\Delta t_i = t_{i+1}-t_i$ in \Eqref{eq:stateless_behavioral_model} is minimized by taking $k=B$. This shows that the optimal breaking policy under this behavioral model is the default one, which does not prompt the user to break.

\paragraph{Evaluation.}
We evaluate the stateless behavioral model using the MovieLens-1M experimental setup, as described in \cref{sec:experiments}. 
We set $\tau=4$ for all users, 
utilize linear regression for engagement prediction, and maintain all hyperparameters without change. Data processing steps are performed as described in \appendixref{app:experimental_details}, and the chosen breaking policies are evaluated.

\paragraph{Results.}
Despite the difference between the true $\mathrm{TPP}$ and our choice of model class,
the LV policy optimization method successfully learned the optimal no-breaks---which in this case, is the policy $p=0$ for all users.
Further detail is provided by \cref{fig:stateless_model_experiment}, which illustrates the policy optimization steps under  $\StatelessTPP(p;u)$ for typical users with unbiased and biased engagement predictions. In both examples, the optimal points of both curves coincide, and the optimal policy is selected despite poor point-wise fit. Combined, these results show that our approach is ``safe'', in the sense that
when breaking is sub-optimal,
the learned breaking policy does not override the default policy.

\begin{figure}[h!]
    \centering
    \includegraphics[width=\linewidth]{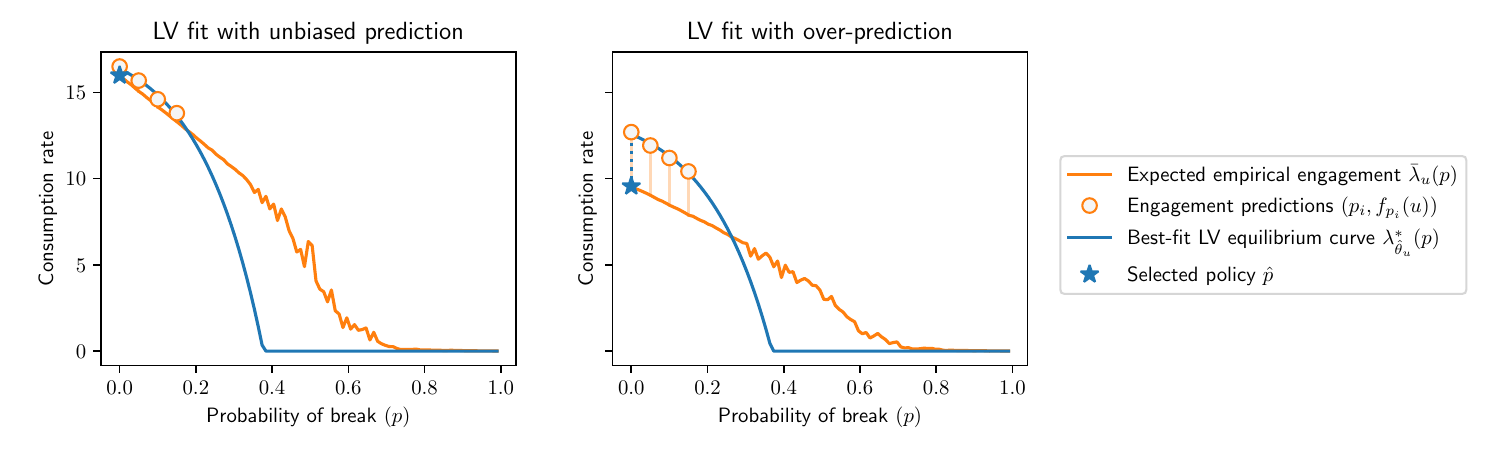}
    \caption{
    Fitting an LV equilibrium curve on the empirical consumption rates of $\StatelessTPP(p;u)$.
    The plots are a realization of the schematic diagram in \cref{fig:policy_optimization_from_data} using data from the MovieLens-1M experiment.
    \textbf{(Left)} Prediction with unbiased engagement predictions.
    \textbf{(Right)} User with extremely low engagement, for which the system tends to over-predict.
    In both cases, the LV policy optimization method selects the optimal policy $p=0$.
    }
    \label{fig:stateless_model_experiment}
\end{figure}

\end{document}